\documentclass[runningheads]{llncs}

\usepackage[mobile]{eccv}

\usepackage{eccvabbrv}

\usepackage{graphicx}
\usepackage{booktabs}
\usepackage{algorithm} 
\usepackage{algpseudocode}
\usepackage{mathrsfs}
\usepackage{wrapfig}
\usepackage{amsmath}

\usepackage[accsupp]{axessibility}  

\usepackage[pagebackref,breaklinks,colorlinks,citecolor=eccvblue]{hyperref}

\makeatletter
\def\blfootnote{\xdef\@thefnmark{}\@footnotetext}
\makeatother

\begin{document}

\title{Simplified Diffusion Schrödinger Bridge} 


\author{
    Zhicong Tang\inst{1*} \and
    Tiankai Hang\inst{2*} \and \\
    Shuyang Gu\inst{3\dagger} \and 
    Dong Chen\inst{3} \and
    Baining Guo\inst{3} \\
    \quad \\
    $^{1}$Tsinghua University\quad$^{2}$Southeast University\quad$^{3}$Microsoft Research Asia \\
    \quad \\
    {\tt\small tzc21@mails.tsinghua.edu.cn, tkhang@seu.edu.cn, \{shuyanggu,doch,bainguo\}@microsoft.com}
}

\authorrunning{Z. Tang et al.}

\institute{
}

\maketitle

\blfootnote{*Equal contribution.~~$\dagger$Corresponding author.}

\begin{abstract}
This paper introduces a novel theoretical simplification of the Diffusion Schrödinger Bridge (DSB) that facilitates its unification with Score-based Generative Models (SGMs), addressing the limitations of DSB in complex data generation and enabling faster convergence and enhanced performance. By employing SGMs as an initial solution for DSB, our approach capitalizes on the strengths of both frameworks, ensuring a more efficient training process and improving the performance of SGM. We also propose a reparameterization technique that, despite theoretical approximations, practically improves the network's fitting capabilities. Our extensive experimental evaluations confirm the effectiveness of the simplified DSB, demonstrating its significant improvements. We believe the contributions of this work pave the way for advanced generative modeling.
  \keywords{Schrödinger Bridge \and Score-based Generative Models}
\end{abstract}

\section{Introduction}
\label{sec:intro}

Score-based Generative Models (SGMs)~\cite{song2019generative,ho2020denoising,song2020score} have recently achieved remarkable success, as highlighted in recent literature~\cite{dhariwal2021diffusion,ramesh2022hierarchical,ramesh2021zero,rombach2022high,gu2022vector}. Despite their advancements, SGMs necessitate a carefully crafted schedule for adding noise, tailored to diverse tasks~\cite{karras2022elucidating,lin2024common,gu2022f,chen2023importance}. Without this customization, training SGMs to handle complex data types, including video~\cite{blattmann2023stable,ho2022imagen,sora2024} and 3D contents~\cite{jun2023shap,tang2023volumediffusion,nichol2022point}, may present significant difficulties. Besides, SGMs are limited to modeling the target data distribution based on a predefined known distribution, \eg Gaussian~\cite{song2020score,ho2020denoising}, which narrows their applicability in certain scenarios, such as conditional generation tasks like unpaired domain transfer~\cite{zhu2017unpaired}.

The Schrödinger Bridge (SB) problem~\cite{schrodinger1932theorie,follmer1988random,chen2020optimal} is considered a more generalized framework that construct a transition between two arbitrary distributions. Numerous approximations of SB have been proposed~\cite{shi2024diffusion,bernton2019schr,chen2016entropic,finlay2020learning,caluya2021wasserstein,pavon2021data}, and one theoretical solution is the Iterative Proportional Fitting (IPF)~\cite{fortet1940resolution,kullback1968probability,ruschendorf1995convergence}. Nevertheless, practical application remains challenging because it involves optimizing joint distributions, which can be highly complex. The previous work IPML~\cite{vargas2021solving} simplifies IPF, and proves that solving the Schrödinger Bridge is equivalent to an autoregressive maximum likelihood estimation objective. A recent work, Diffusion Schrödinger Bridge (DSB)~\cite{de2021diffusion}, also simplifies this by approximating the joint distribution optimization as a conditional distribution optimization problem. It employs two neural networks to enable the transition from one distribution to another by iteratively training them.

While DSB possesses theoretical superiority compared to SGMs, its practical implementation is hindered by a slow convergence rate and potential issues with network fitting capabilities~\cite{de2021diffusion}. Moreover, distinct training methodologies between DSB and SGM currently prevent the former from leveraging the rapid advancements being made in the field of SGM.

In this paper, we bridge the gap between the DSB and SGM by formulating a simplified optimization objective for DSB. Through this unification, SGM can be served as an initial solution for DSB, enabling enhanced outcomes with further iterations according to DSB~\cite{de2021diffusion}. Our theoretical analysis reveals that this initialization strategy is pivotal for the training of DSB, addressing the issue of slow convergence and yielding improved performance. Simultaneously, from the perspective of SGM, a pre-trained model can experience consistent enhancement through supplementary training using the DSB approach.

Furthermore, a key to the success of SGM~\cite{dhariwal2021diffusion,rombach2022high,ramesh2022hierarchical} and other generative models, \eg Flow Matching (FM)~\cite{lipman2022flow}, Conditional Flow Matching (CFM)~\cite{tong2023improving} and Rectified Flow~\cite{liu2022flow}, lies in their use of reparameterization~\cite{salimans2022progressive,ho2020denoising} to create a consistent output space across different timesteps, which maximizes the shared network's fitting capabilities. Drawing inspiration from them, we propose a reparameterization trick to standardize the output space for the DSB. Despite the reliance on a considerable number of theoretical approximations, we were pleasantly surprised to discover that this also unlocked the network's potential, subsequently speeding up convergence and leading to enhanced results.

We conduct comprehensive experiments to confirm the effectiveness of our proposed simplification and to showcase the significant enhancements it brings to DSB. Our contributions are summarized as follows:
\begin{itemize}
    \item[•] We introduce a theoretical simplification of DSB, proving its equivalence to the original formulation.
    \item[•] By employing the simplified version, we successfully integrate SGM with DSB. Using SGM as the initialization for DSB significantly accelerates its convergence.
    \item[•] We devise a reparameterization technique that, despite some theoretical approximations, practically enhances the network's fitting capabilities.
    \item[•] Through extensive experimentation, we examine practical training for DSB and validate the effectiveness of our proposed simplification.
\end{itemize}

\section{Preliminary}

In this section, we initiate by recalling some essential preliminaries of SGM in Section~\ref{sec:SGM}. Then, we introduce the SB problem and DSB, an approximate solution of it, in Section~\ref{sec:SB_and_DSB}. Finally, in Section~\ref{sec:general_perspective}, we summarized prevalent dynamic generative models and unified them into the SDE form of SB.

\noindent \textbf{Notation} Consider a data distribution $p_\textup{data}$ and a prior distribution $p_\textup{prior}$, each defined over the $\mathbb{R}^d$ with positive density. We denote the transformation from $p_\textup{data}$ to $p_\textup{prior}$ as the forward process, whereas the reverse transformation as the backward process. The symbols $k \in \{0, \dots,N\}$ and $t \in [0,T]$ represent discrete and continuous time horizons, respectively. In discrete-time scenarios, $\mathscr{P}_l=\mathscr{P}\left((\mathbb{R}^d)^l\right)$ signifies the probability of a $l$-state joint distribution for any $l\in\mathbb{N}$. In continuous-time scenarios, we use $\mathscr{P}(\mathcal{C})$ where $\mathcal{C}=\textup{C}\left([0,T],\mathbb{R}^d\right)$. $\textup{H}(p)=-\int p(x)\log p(x)\mathrm{d}x$ denotes the entropy of $p$. The Kullback-Leibler divergence between distributions $p$ and $q$ is defined by $\textup{KL}(p|q)=\int p(x)\log \frac{p(x)}{q(x)}\mathrm{d}x$, and Jeffrey's divergence is denoted as $\textup{J}(p|q)=\frac{1}{2}\left(\textup{KL}(p|q)+\textup{KL}(q|p)\right)$.

\subsection{Score-based Generative Models}
\label{sec:SGM}

Score-based Generative Models (SGMs)~\cite{song2019generative,ho2020denoising,song2020denoising,sohl2015deep} connect two distributions through a dual process: a forward process that transitions the data distribution, $p_{\text{data}}$, toward a prior distribution, $p_{\text{prior}}$, and a reverse process, typically guided by neural networks, that converts the prior back to the data distribution. These two processes can be modeled as Markov chains. Given an initial data distribution $p_{\textup{data}}$ and a target prior distribution $p_{\textup{prior}}$, the forward process $p_{k+1|k}(x_{k+1}|x_k)$ is designed to transition from $p_0 = p_{\textup{data}}$ step-by-step towards a close approximation of $p_{\textup{prior}}$. This process generates a sequence $x_{0:N}$ from the $(N+1)$ intermediate steps. This trajectory's joint probability density is then formally defined as:

\begin{equation}
    p(x_{0:N})=p_0(x_0)\prod_{k=0}^{N-1}p_{k+1|k}(x_{k+1}|x_k).
    \label{eq:forward_chain}
\end{equation}

Through the backward process, the joint density can also be reformulated as a time-reversed distribution:

\begin{equation}
    \begin{aligned}
        & p(x_{0:N})=p_N(x_N)\prod_{k=0}^{N-1}p_{k|k+1}(x_k|x_{k+1}) \\
        \text{where}\ & p_{k|k+1}(x_k|x_{k+1})=\frac{p_{k+1|k}(x_{k+1}|x_k)p_k(x_k)}{p_{k+1}(x_{k+1})},
    \end{aligned}
    \label{eq:backward_chain}
\end{equation}

However, directly computing $p_{k|k+1}(x_k|x_{k+1})$ is typically challenging. SGM utilizes a simplified approach that regard the forward process as a gradual adding of Gaussian noise:

\begin{equation}
    p_{k+1|k}(x_{k+1}|x_k)=\mathcal{N}(x_{k+1};x_k+\gamma_{k+1}f_k(x_k),2\gamma_{k+1}\mathbf{I}).
    \label{eq:forward_gaussian}
\end{equation}

\noindent It follows that for a sufficiently extensive $N$, the distribution $p_N$ will converge to  Gaussian distribution, which we set as $p_{\textup{prior}}$. Moreover, the temporal inversion in Equation~\ref{eq:backward_chain} can be analytically approximated~\cite{hyvarinen2005estimation,vincent2011connection,anderson1982reverse} as

\begin{equation}
    \begin{aligned}
        p_{k|k+1}(x_k|x_{k+1})&=p_{k+1|k}(x_{k+1}|x_k)\exp\left(\log p_k(x_k)-\log p_{k+1}(x_{k+1})\right)\\
        &\approx\mathcal{N}(x_k;x_{k+1}-\gamma_{k+1}f_{k+1}(x_{k+1})+2\gamma_{k+1}\nabla\log p_{k+1}(x_{k+1}),2\gamma_{k+1}\mathbf{I})
    \end{aligned}
    \label{eq:backward_gaussian}
\end{equation}

\noindent under the assumption that $p_k\approx p_{k+1}$ and $f_k(x_k)\approx f_{k+1}(x_{k+1})$. Subsequently, SGM employs neural networks $s_\theta(x_{k+1},k+1)$ to approximate the score term $\nabla\log p_{k+1}(x_{k+1})$, thus the reverse process can be effectively modeled. By sampling $x_N \sim p_{\textup{prior}}$, followed by iteratively applying ancestral sampling via $x_k \sim p_{k|k+1}(x_k|x_{k+1})$, culminating in the estimation of $x_0 \sim p_{\textup{data}}$.

The diffusion and denoising processes can also be formulated as continuous-time Markov chains. The forward process can be represented as a continuous-time Stochastic Differential Equation (SDE)~\cite{song2020score}:

\begin{equation}
    \mathrm{d}\mathbf{X}_t=f_t(\mathbf{X}_t)\mathrm{d}t+\sqrt{2}\mathrm{d}\mathbf{B}_t,
    \label{eq:forward_SDE}
\end{equation}

\noindent where $f: \mathbb{R}^d \to \mathbb{R}^d$ is the drift term and $(\mathbf{B}_t)_{t \in [0,T]}$ denotes the Brownian motion, we called it the diffusion term. We note that the Markov chain with transition kernel (\ref{eq:forward_gaussian}) can be viewed as the Euler-Maruyama discretization of this SDE. The backward process, on the other hand, is conceptualized by solving the time-reversed SDE~\cite{haussmann1986time,follmer1985entropy,cattiaux2021time} which corresponds to (\ref{eq:backward_gaussian}):

\begin{equation}
    \mathrm{d}\mathbf{Y}_t=\left(-f_t(\mathbf{Y}_t)+2\nabla\log p_{T-t}(\mathbf{Y}_t)\right)\mathrm{d}t+\sqrt{2}\mathrm{d}\mathbf{B}_t.
    \label{eq:backward_SDE}
\end{equation}

\subsection{Schrödinger Bridge and Diffusion Schrödinger Bridge}
\label{sec:SB_and_DSB}

Consider a reference density $p_\textup{ref} \in \mathscr{P}_{N+1}$ given by Equation~\ref{eq:forward_chain}, the Schrödinger Bridge (SB) problem aims to find $\pi^* \in \mathscr{P}_{N+1}$ which satisfies

\begin{equation}
    \pi^*=\arg\min\{\textup{KL}(\pi|p_\textup{ref}):\pi\in\mathscr{P}_{N+1},\pi_0=p_\textup{data},\pi_N=p_\textup{prior}\}.
    \label{eq:SB_definition}
\end{equation}

Upon acquiring the optimal solution $\pi^*$, we can sample $x_0\sim p_\textup{data}$ by initially drawing $x_N\sim p_\textup{prior}$ and iterates the ancestral sampling $\pi_{t|t+1}(x_t|x_{t+1})$. Conversely, the sampling of $x_N\sim p_\textup{prior}$ is also feasible via the commencement of $x_0\sim p_\textup{data}$ followed by $ \pi_{t+1|t}(x_{t+1}|x_t)$. SB facilitates a bidirectional transition that is not predicated upon any presuppositions regarding $p_\textup{data}$ and $p_\textup{prior}$.

Generally, the SB problem lacks a closed-form solution. Researchers employ the Iterative Proportional Fitting (IPF)~\cite{fortet1940resolution,kullback1968probability,ruschendorf1993note} to address it through iterative optimization:

\begin{equation}
    \label{eq:IPF_definition}
    \begin{aligned}
        \pi^{2n+1}&=\arg\min\{\textup{KL}(\pi|\pi^{2n}):\pi\in\mathscr{P}_{N+1},\pi_N=p_\textup{prior}\},\\
        \pi^{2n+2}&=\arg\min\{\textup{KL}(\pi|\pi^{2n+1}):\pi\in\mathscr{P}_{N+1},\pi_0=p_\textup{data}\}.
    \end{aligned}
\end{equation}

\noindent where $\pi^0=p_\textup{ref}$ is the initialization condition. Nevertheless, the IPF method needs to compute and optimize the joint density, which is usually infeasible within practical settings due to its computational complexity.

Diffusion Schrödinger Bridge (DSB)~\cite{de2021diffusion} can be conceptualized as an approximate implementation of IPF. It dissects the optimization of the joint density into a series of conditional density optimization problems. Specifically, $\pi$ is disentangled into $\pi_{t+1|t}$ and $\pi_{t|t+1}$ in the forward and backward trajectory, respectively.

\begin{equation}
    \label{eq:DSB_definition}
    \begin{aligned}
        \pi^{2n+1}&=\arg\min\{\textup{KL}(\pi_{k|k+1}|\pi^{2n}_{k|k+1}):\pi\in\mathscr{P}_{N+1},\pi_N=p_\textup{prior}\},\\
        \pi^{2n+2}&=\arg\min\{\textup{KL}(\pi_{k+1|k}|\pi^{2n+1}_{k+1|k}):\pi\in\mathscr{P}_{N+1},\pi_0=p_\textup{data}\}.
    \end{aligned}
\end{equation}

We can verify that when the conditional densities (Equation~\ref{eq:DSB_definition}) is optimized, the joint distribution (Equation~\ref{eq:IPF_definition}) can also be optimized. To optimize it, DSB follows the common practice of SGM~\cite{song2019generative,ho2020denoising} to assume the conditional distributions $\pi_{t+1|t}$ and $\pi_{t|t+1}$ are Gaussian distributions. It allows DSB to analytically calculate the time reversal process. Following this, DSB employs two separate neural networks to model the forward and backward transitions, respectively. Through a series of mathematical derivation and approximation~\cite{de2021diffusion}, the training loss of DSB is resembled as

\begin{equation}
    \label{eq:DSB_original_loss}
    \begin{aligned}
        &\mathcal{L}_{B_{k+1}^n}=\mathbb{E}_{(x_k,x_{k+1})\sim p_{k,k+1}^n}\left[\left\|B_{k+1}^n(x_{k+1})-\left(x_{k+1}+F_k^n(x_k)-F_k^n(x_{k+1})\right)\right\|^2\right]\\
        &\mathcal{L}_{F_k^{n+1}}=\mathbb{E}_{(x_k,x_{k+1})\sim q_{k,k+1}^n}\left[\left\|F_{k}^{n+1}(x_{k})-\left(x_{k}+B_{k+1}^n(x_{k+1})-B_{k+1}^n(x_{k})\right)\right\|^2\right]
    \end{aligned}
\end{equation}

\noindent where $p^n=\pi^{2n}$, $q^n=\pi^{2n+1}$ denotes the forward and backward joint densities, respectively. $p_{k+1|k}^n(x_{k+1}|x_{k}) = \mathcal{N}(x_{k+1};x_k+\gamma_{k+1}f_{k}^n(x_{k}), 2\gamma_{k+1} I)$ is the forward process and $q_{k|k+1}^n(x_k|x_{k+1}) = \mathcal{N}(x_k;x_{k+1}+\gamma_{k+1}b_{k+1}^n(x_{k+1}), 2\gamma_{k+1} I)$ is the backward process, where $f_{k}^n(x_{k})$ and $b_{k+1}^n(x_{k+1})$ are drift terms.

In practice, DSB uses two neural networks to approximate $B_{\beta^n}(k,x)\approx B_{k}^n(x)=x+\gamma_kb_k^n(x)$ and $F_{\alpha^n}(k,x)\approx F_{k}^n(x)=x+\gamma_{k+1}f_{k}^n(x)$, $\alpha$ and $\beta$ denotes the network parameters. With the initialization $p_0=p_\textup{ref}$ manually pre-defined as Equation~\ref{eq:forward_gaussian}, DSB iteratively optimizes the backward network $B_{\beta^n}$ and the forward network $F_{\alpha^n}$ for $n\in\{0,1,\ldots,L\}$ to minimize Equation~\ref{eq:DSB_definition}. For $(2n+1)$-th epoch of DSB, we refer to the optimization of the backward network $B_{\beta^n}$, and we optimize the forward network in $F_{\alpha^n}$ with $(2n+2)$-th epoch. ~\cite{de2021diffusion} proves the convergence of this approach:

\begin{proposition}
    Assume $p_N>0$, $p_\textup{prior}>0$, $\left|\textup{H}(p_\textup{prior})\right|<+\infty$,\\$\int_{\mathbb{R}^d}\left|\log p_{N|0}(x_N|x_0)\right|p_\textup{data}(x_0)p_\textup{prior}(x_N)\mathrm{d}x_0\mathrm{d}x_N<+\infty$. Then $(\pi^n)_{n\in\mathbb{N}}$ is well-defined and for any $n>1$ we have 
    \begin{equation*}
        \textup{KL}(\pi^{n+1}|\pi^n)<\textup{KL}(\pi^{n-1}|\pi^n),\quad\textup{KL}(\pi^{n}|\pi^{n+1})<\textup{KL}(\pi^{n}|\pi^{n-1})
    \end{equation*}
     In addition, $\left(\left\|\pi^{n+1}-\pi^n\right\|_\textup{TV}\right)_{n\in\mathbb{N}}$ and $\left(\textup{J}(\pi^{n+1},\pi^n)\right)_{n\in\mathbb{N}}$ are non-increasing. Finally, we have $\lim_{n\to+\infty}n\{\textup{KL}(\pi_0^{n}|p_\textup{data})+\textup{KL}(\pi_N^{n}|p_\textup{prior})\}=0$.
    \label{prop:dsb_convergence}
\end{proposition}

\subsection{A General Perspective of Dynamic Generative Models}
\label{sec:general_perspective}

Previous works~\cite{schrodinger1932theorie,leonard2014survey} have pointed out the optimal solution of Schrödinger Bridge follows the form of SDE:

\begin{equation}
    \label{eq:SB_SDE}
    \begin{aligned}
        \mathrm{d}\mathbf{X}_t&=\left(f(\mathbf{X}_t,t)+g^2(t)\nabla\log\mathbf{\Psi}(\mathbf{X}_t,t)\right)\mathrm{d}t+g(t)\mathrm{d}\mathbf{W}_t,X_0\sim p_\textup{data}\\
        \mathrm{d}\mathbf{X}_t&=\left(f(\mathbf{X}_t,t)-g^2(t)\nabla\log\mathbf{\hat{\Psi}}(\mathbf{X}_t,t)\right)\mathrm{d}t+g(t)\mathrm{d}\mathbf{\bar{W}}_t,X_T\sim p_\textup{prior}
    \end{aligned}
\end{equation}

\noindent where $\mathbf{W}_t$ is Wiener process and $\mathbf{\bar{W}}_t$ its time reversal. $\mathbf{\Psi},\mathbf{\hat{\Psi}}\in\textup{C}^{2,1}\left([0,T],\mathbb{R}^d\right)$ are time-varying energy potentials that constrained by the interconnected PDEs:

\begin{equation}
    \begin{aligned}
        \left\{\begin{matrix}
        \frac{\partial\mathbf{\Psi}}{\partial t}=&-\nabla_x\mathbf{\Psi}^\top f-\frac{1}{2}\textup{Tr}\left(g^2\nabla_x^2\mathbf{\Psi}\right)\\
        \frac{\partial\mathbf{\hat{\Psi}}}{\partial t}=&-\nabla_x\cdot(\mathbf{\hat{\Psi}}f)+\frac{1}{2}\textup{Tr}\left(g^2\nabla_x^2\mathbf{\hat{\Psi}}\right)
        \end{matrix}\right. \quad\quad\quad\quad\ \  \\
        \textup{s.t.}\ \mathbf{\Psi}(x,0)\mathbf{\hat{\Psi}}(x,0)=p_\textup{data}, \mathbf{\Psi}(x,T)\mathbf{\hat{\Psi}}(x,T)=p_\textup{prior}.
    \end{aligned}
    \label{eq:SB_PDE}
\end{equation}

\noindent More generally, for the distribution of SB at time $t$, we can acheived by: 
\begin{equation}
    p_t = \mathbf{\Psi}(x,t)\mathbf{\hat{\Psi}}(x,t)    
    \label{eq:SB_marginal}
\end{equation}

Upon close inspection, one can observe that Equation~\ref{eq:forward_SDE} and Equation~\ref{eq:SB_SDE} differ only by the additional non-linear drift term $g^2(t)\nabla\log\mathbf{\Psi}(\mathbf{X}_t,t)$. 
Notably, SGMs may be regarded as a particular instantiation of DSB when the non-linear drift terms are set to zero, \ie $\mathbf{\Psi}(\mathbf{X}_t,t)\equiv C$. For Variance Preserving (VP)~\cite{ho2020denoising} and Variance Exploding (VE)~\cite{song2019generative} noise schedule in SGMs, $f(\mathbf{X}_t,t) = -\alpha_t\mathbf{X}_t$ where $\alpha_t\in\mathbb{R}_{\ge 0}$, and the denoising model is essentially solving Equation~\ref{eq:SB_SDE} using a learnable network.

More general, other dynamic generative models, such as Flow Matching (FM)~\cite{lipman2022flow}, I$^2$SB~\cite{liu20232}, Bridge-TTS~\cite{chen2023schrodinger}, can be encapsulated within the framework of Equation~\ref{eq:SB_SDE} by selecting appropriate functions for $f(\mathbf{X}_t,t)$ and $\mathbf{\Psi}(\mathbf{X}_t,t)$. We leave a more comprehensive discussion in the appendix. This unification of dynamic generative models suggests the possibility of a more profound linkage between DSB, SGM and other dynamic generative models, which we will explore in the following sections. 

\section{Simplified Diffusion Schrödinger Bridge}

In this section, we start by analyzing the training function of DSB. We introduce a streamlined training objective and establish its equivalence to the original form in Section~\ref{sec:simplified_target}. In Section~\ref{sec:convergence_analysis}, we analyze the convergence of DSB and point out that initialization is the key for optimization. Leveraging our proposed simplified objective, we are able, for the first time, to treat SGMs as the backward solution of the DSB's initial epoch, which will be elucidated in Section~\ref{sec:diffusion_as_dsb}. In Section~\ref{sec:initialization}, we explore how to employ an SGM or FM model as a suitable initialization for the first forward epoch in DSB.

\subsection{Simplified training target}
\label{sec:simplified_target}

In vanilla DSB, the training target for both forward and backward models is complex, as shown in Equation~\ref{eq:DSB_original_loss}, which is hard to understand its physical meaning. Firstly, we propose a simplified version of training loss for DSB as

\begin{equation}
    \begin{aligned}
        &\mathcal{L}_{B_{k+1}^n}^{'}=\mathbb{E}_{(x_k,x_{k+1})\sim p_{k,k+1}^n}\left[\left\|B_{k+1}^n(x_{k+1})-x_{k}\right\|^2\right], \\
        &\mathcal{L}_{F_k^{n+1}}^{'}=\mathbb{E}_{(x_k,x_{k+1})\sim q_{k,k+1}^n}\left[\left\|F_{k}^{n+1}(x_{k})-x_{k+1}\right\|^2\right].
    \end{aligned}
    \label{eq:DSB_simplified_loss}
\end{equation}

The following proposition demonstrates that the Simplified DSB (S-DSB) is equivalent to the original one shown in Equation~\ref{eq:DSB_original_loss}.

\begin{proposition}
    \label{prop:SDSB}
    Assume that for any $n\in\mathbb{N}$ and $k\in\{0,1,\ldots,N-1\}$, $\gamma_{k+1}>0$, $q_{k|k+1}^n=\mathcal{N}\left(x_k;B_{k+1}^n(x_{k+1}),2\gamma_{k+1}\mathbf{I}\right)$, $p_{k+1|k}^n=\mathcal{N}\left(x_{k+1};F_{k}^n(x_{k}),2\gamma_{k+1}\mathbf{I}\right)$, where $B_{k+1}^n(x)=x+\gamma_{k+1}b_{k+1}^n(x)$ and $F_{k}^n(x)=x+\gamma_{k+1}f_{k}^n(x)$. Assume $b_{k+1}^n(x_{k+1})\approx b_{k+1}^n(x_k)$ and $f_{k}^n(x_{k+1})\approx f_{k}^n(x_k)$. Then we have
    \begin{equation}
        \begin{aligned}
            \mathcal{L}_{B_{k+1}^n}^{'}\approx\mathcal{L}_{B_{k+1}^n}, 
            \mathcal{L}_{F_k^{n+1}}^{'}\approx\mathcal{L}_{F_k^{n+1}}.
        \end{aligned}
        \label{eq:propositon_loss_equivalent}
    \end{equation}
\end{proposition}

We leave the proof in the appendix. It is worth noting that this approximation is predicated on the assumptions that $b_{k+1}^n(x_{k+1})\approx b_{k+1}^n(x_{k})$ and $f_k^n(x_k)\approx f_k^n(x_{k+1})$, which were previously employed in the original formulation~\cite{de2021diffusion} of DSB (Equation~\ref{eq:DSB_original_loss}) to derive the reverse-time transition. Hence, the approximation introduced by this simplification is theoretically acceptable. 

This simplified loss has two advantages. First, it saves half of the number of forward evaluations (NFE) when computing the prediction target. The original loss Equation~\ref{eq:DSB_original_loss} needs to run model twice ($F_k^n(x_k)$ and $F_k^n(x_{k+1})$), while the simplified target in Equation~\ref{eq:DSB_simplified_loss} needs only one evaluation. This may be critical when the networks $B_{k+1}^n(x)$ and $F_k^n(x)$ are large or the dimension of $x$ is high in practical settings, e.g., image and video generation.

Secondly, and of greater significance, the simplified optimization objective aligns with SGM, enhancing our intuitive grasp of the process. Whether dealing with a forward or backward process, the requirement is merely to predict the subsequent state, akin to the approach of SGM. In the following section, we will see that this unification is crucial for the effective training of DSB.

\subsection{The convergence analysis}
\label{sec:convergence_analysis}

We illustrate the framework of DSB in Figure~\ref{fig:dsb_method}. Note that DSB can transit between two unknown distributions. In order to be consistent with SGM, we use $p_\textup{data}$ and $p_\textup{prior}$ to refer to these two distributions. Specifically, for each epoch, the backward model is trained to map from $p_\textup{prior}$ back to $p_\textup{data}$. Subsequently, the forward model is trained to approximate the mapping from $p_\textup{data}$ to $p_\textup{prior}$.

\begin{figure}[t]
  \centering
   \includegraphics[width=0.8\linewidth]{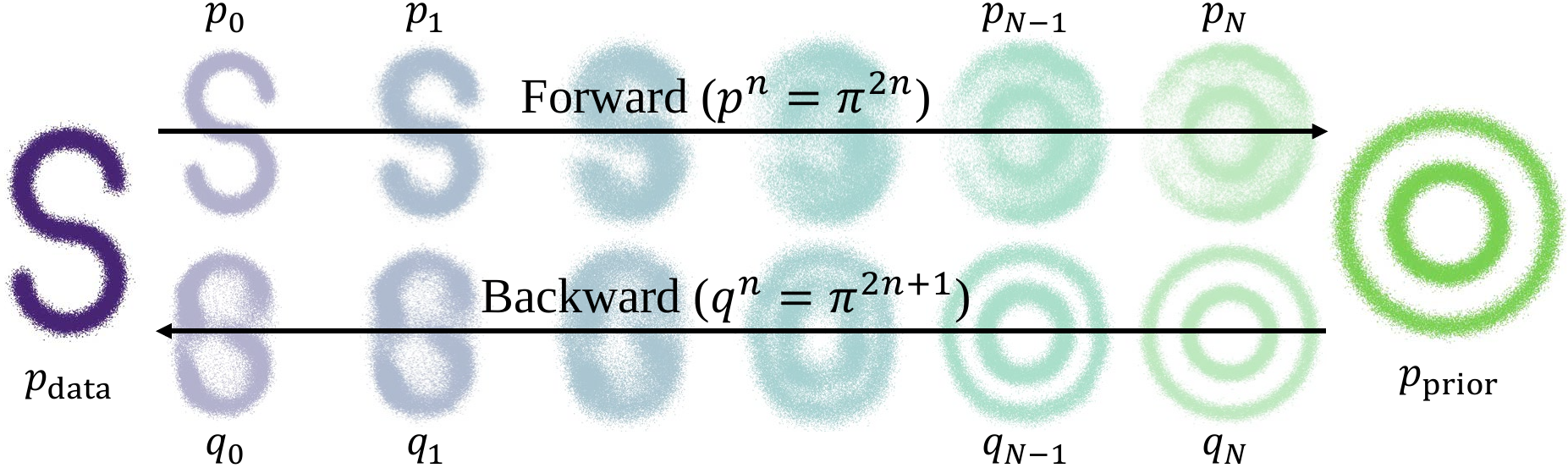}
   \caption{Illustration of the pipeline of Diffusion Schrödinger Bridge.}
   \label{fig:dsb_method}
   \vspace{-0.5cm}
\end{figure}

During practical training, we observed that the results of DSB frequently exhibit slow convergence, requiring an excessive number of epochs to reach optimal state. To elucidate the underlying cause, we conducted an analysis of the convergence performance and take the backward process as an example. Note that $p_\textup{data}=\pi_0^{2n}=p_0^n=\int p_N(x_N)\prod_{i=0}^{N-1}p_{i|i+1}^n(x_i|x_{i+1})\mathrm{d}x_i$, we have

\begin{equation}
    \begin{aligned}
        \pi_0^{2n+1} - p_\textup{data} &= q_{0}^n(x_0) - p_\textup{data} \\
        &= \int q_{0|1}^n(x_0|x_1)q_{1}^n(x_1)\mathrm{d}x_1 - p_\textup{data} \\ 
        &= \cdots \\
        &= \int q_{N}^n(x_N)\prod_{i=0}^{N-1}q_{i|i+1}^n(x_i|x_{i+1})\mathrm{d}x_i - \int p^n_N(x_N)\prod_{i=0}^{N-1}p_{i|i+1}^n(x_i|x_{i+1})\mathrm{d}x_i\\
        &\stackrel{\text{*}}{=} \int (p_\textup{prior}(x_N) - p^n_N(x_N)) \prod_{i=0}^{N-1}p_{i|i+1}^n(x_i|x_{i+1})\mathrm{d}x_i, \\
    \end{aligned}
\end{equation}

\noindent where $(*)$ is established because DSB assumes every epoch achieves the convergence state (Equation~\ref{eq:DSB_definition}), \ie $q_{i|i+1}^n(x_i|x_{i+1})=p_{i|i+1}^n(x_i|x_{i+1})$ for any $i\in\{0,1,\ldots,N-1\}$. Similarly, we can get

\begin{equation}
    \begin{aligned}
        \pi_N^{2n+2}-p_\textup{prior}=\int \left(q_\textup{data}(x_0)-q^n_0(x_0)\right)\prod_{i=1}^{N}q_{i+1|i}^n(x_{i+1}|x_{i})\mathrm{d}x_i.
    \end{aligned}
\end{equation}

Consequently, our findings indicate that the convergence of $(n+1)$-th epoch relies on the convergence of previous $n$-th epoch. This observation underscores the stability of DSB when each epoch reaches the optimal solution. Moreover, these insights suggest that an effective initial state may significantly expedite the convergence process.

\subsection{SGM as the first backward epoch}
\label{sec:diffusion_as_dsb}
As previously noted, our simplified DSB and SGM share the same training objective. Hence, by setting the $p_\textup{ref}$ in Equation~\ref{eq:SB_definition} the same as the noise schedule of SGM, the first epoch of DSB is theoretically equivalent to the training of SGM. Therefore, we can utilize a pre-trained SGM as the solution of the first backward epoch ($B_{\beta^1}$). Through multiple rounds of forward and backward training, the quality of the generated data will be continuously improved according to the progressively enhancing attributes of DSB (as stated in Proposition~\ref{prop:dsb_convergence}). Consequently, this process yields a new model that surpasses the initial SGM in terms of modeling the target distribution.

It's noteworthy that the majority of current SGMs~\cite{rombach2022high,gu2022vector} employ the reparameterization trick~\cite{ho2020denoising}, whereby the denoising model typically predicts either the noise ($\epsilon \in p_\textup{prior}$) or the clean data ($x_0 \in p_\textup{data}$). We adapt this approach to align with our simplified optimization target in Equation~\ref{eq:DSB_simplified_loss} by reverting the prediction focus to the next state~\cite{ho2020denoising}. As for Flow Matching~\cite{lipman2022flow} models, it can be considered as a variant of SGM that utilizes a unique noise schedule ($f(\mathbf{X}_t,t) = -\frac{t}{T}\mathbf{X}_t$) and is reparameterized to produce outputs based on $x_0 - \epsilon$ via neural networks. By recalibrating these models, we can seamlessly integrate them as the initial backward epoch in the training process of DSB.

\subsection{Initialization of the first forward epoch}
\label{sec:initialization}
The training of the first forward epoch is the second epoch in the entire training, which is designed to train a neural network ($F_{\alpha^1}$) to transform $p_\textup{data}$ into $p_\textup{prior}$. It ensures that the intermediate states adhere to the KL constraint imposed by trajectories of the first backward epoch, as described in Equation~\ref{eq:DSB_original_loss}. Rather than training from scratch, it is advantageous to leverage a pretrained SGM as a starting point and refine it according to Equation~\ref{eq:DSB_simplified_loss}. Specifically, we directly train an SGM to facilitate the transition from $p_\textup{data}$ to $p_\textup{prior}$, and shift the training objective to the subsequent state. Although standard SGMs typically model the transition from the prior distribution to the data distribution, we find that by utilizing FM models and reversing the roles of $p_\textup{prior}$ and $p_\textup{data}$, we can also approximate a rough transformation from the data distribution to the prior distribution. This approximation serves as an effective initial state for training the first forward epoch of DSB.

Additionally, an alternative strategy involves utilizing the same SGM we used in the first backward epoch but adopts an altered reparameterization. Unlike the backward process that requires the prediction of $\mathbb{E}(x_{k-1})$, the forward process focuses on estimating $\mathbb{E}(x_{k+1})$. To facilitate this shift, we adjust the reparameterization coefficients accordingly. For instance, in Variance Preserving~\cite{ho2020denoising} SGM, we can employ $\sqrt{\bar{\alpha}_{k+1}}x_0+\sqrt{1-\bar{\alpha}_{k+1}} \epsilon$ to initialize the result, where $x_0$ is the prediction made by the denoising network, and $\epsilon$ can be calculated by $\sqrt{\bar{\alpha}_{k}}x_0+\sqrt{1-\bar{\alpha}_{k}} \epsilon = x_k$. Empirical results have shown that initiating training with this adjusted reparameterization leads to more efficient convergence compared to training from scratch.

\section{Reparameterized Diffusion Schrödinger Bridge}
\label{sec:RDSB}

Score-based generative models (SGMs) employ neural networks $s_\theta(x_{k+1},k+1)$ to approximate the score $\nabla\log p_{k+1}(x_{k+1})$ as described in Equation~\ref{eq:backward_gaussian}. However, contemporary SGMs~\cite{ho2020denoising} tend to predict noise ($\epsilon \in p_\textup{prior}$) or data ($x_0 \in p_\textup{data}$) rather than directly estimating the score. Researchers~\cite{salimans2022progressive} find that a suitable reparameterization is crucial in the training of SGM, since a single network may struggle with the dynamic distribution of score across time~\cite{hang2023efficient,balaji2022ediffi}. Reparameterization simplifies the network's task by allowing it to focus on a consistent output target at different timesteps.

Within the framework of S-DSB, both the forward and backward networks aim to predict the expectation of subsequent states, as depicted in Equation~\ref{eq:DSB_simplified_loss}. This challenge parallels the one encountered in the training of original SGM. To circumvent this difficulty, we introduce Reparameterized DSB (R-DSB).

Similar to Denoising Diffusion Probabilistic Models~\cite{ho2020denoising}, we first derive an estimation for the posterior transitions $p_{k|k+1,0}^n(x_{k}|x_{k+1},x_0)$ and $q_{k+1|k,N}^n(x_{k+1}|x_k,x_N)$, where $p^n=\pi^{2n}$ and $q^n=\pi^{2n+1}$ are defined in Equation~\ref{eq:DSB_definition}. In the subsequent discussion, we denote $\bar{\gamma}_k=\sum_{n=1}^k\gamma_{n}$ and $\bar{\gamma}_0=0$.

\begin{proposition}
    Assume $\sum_{k=1}^N\gamma_{k}=1$. Given the forward and backward trajectories $x_{0:N}\sim p^n(x_{0:N})=p_0^n(x_0)\prod_{k=0}^{N-1}p_{k+1|k}^n(x_{k+1}|x_k)$ and $x_{0:N}'\sim q^n(x_{0:N})=q_N^n(x_N)\prod_{k=0}^{N-1}p_{k|k+1}^n(x_k|x_{k+1})$. Under some assumptions, we have
    \begin{equation}
        \begin{aligned}
            & q_{k|k+1}^n(x_{k}|x_{k+1})\approx p_{k|k+1,0}^n(x_{k}|x_{k+1},x_0) = \mathcal{N}(x_{k};\mu_{k+1}^n(x_{k+1},x_0),\sigma_{k+1}\mathbf{I}), \\
            & p_{k+1|k}^{n+1}(x_{k+1}|x_{k})\approx q_{k+1|k,N}^n(x_{k+1}'|x_k',x_N') = \mathcal{N}(x_{k+1}';\tilde{\mu}_k^n(x_k',x_N'),\tilde{\sigma}_{k+1}\mathbf{I}),
        \end{aligned}
        \label{eq:RDSB_1}
    \end{equation}
    \begin{equation}
        \begin{aligned}
            & \mu_{k+1}^n(x_{k+1},x_0)\approx x_{k+1}+\frac{\gamma_{k+1}}{\bar{\gamma}_{k+1}}(x_0-x_{k+1}), \sigma_{k+1}=\frac{2\gamma_{k+1}\bar{\gamma}_{k}}{\bar{\gamma}_{k+1}}, \\
            & \tilde{\mu}^n_k(x_k',x_N')\approx x_k'+\frac{\gamma_{k+1}}{1-\bar{\gamma}_{k}}(x_N'-x_k'), \tilde{\sigma}_{k+1}=\frac{2\gamma_{k+1}(1-\bar{\gamma}_{k+1})}{1-\bar{\gamma}_k}.
        \end{aligned}
        \label{eq:RDSB_2}
    \end{equation}
    \label{prop:RDSB_prop}
\end{proposition}

\noindent We leave the proof in the appendix. The proposition suggests that when learning the forward process, $x_{k+1}$ depends on the previous state $x_k$ and the terminal $x_N$. While in the backward process, $x_{k}$ also depends on the previous state $x_{k+1}$ and the terminal $x_0$. These leads to the following reparameterization.

\textbf{Terminal Reparameterized DSB (TR-DSB)} Following DDPM~\cite{ho2020denoising}, we design two neural networks to output the terminal points of the forward and backward trajectories, respectively. The training losses are 

\begin{equation}
    \begin{aligned}
        &\mathcal{L}_{\tilde{B}_{k+1}^n}=\mathbb{E}_{(x_0,x_{k+1})\sim p_{0,k+1}^n}\left[\left\|\tilde{B}_{k+1}^n(x_{k+1})-x_{0}\right\|^2\right], \\
        &\mathcal{L}_{\tilde{F}_{k}^{n+1}}=\mathbb{E}_{(x_{k},x_{N})\sim q_{k,N}^n}\left[\left\|\tilde{F}_{k}^{n+1}(x_{k})-x_{N}\right\|^2\right].
    \end{aligned}
    \label{eq:TRDSB_loss}
\end{equation}

\noindent Upon completion of training, we can perform sampling using Equation~\ref{eq:RDSB_1} and Equation~\ref{eq:RDSB_2}, where $x_N=\tilde{F}_{k}^n(x_{k})$ and $x_0=\tilde{B}_{k+1}^n(x_{k+1})$.

\textbf{Flow Reparameterized DSB (FR-DSB)} Following Flow Matching~\cite{lipman2022flow}, we can also employ neural networks to predict the vector that connects the start point to the end point. The training losses are

\begin{equation}
    \begin{aligned}
        &\mathcal{L}_{\tilde{b}_{k+1}^n}=\mathbb{E}_{(x_0,x_{k+1})\sim p_{0,k+1}^n}\left[\left\|\tilde{b}_{k+1}^n(x_{k+1})-\frac{x_0-x_{k+1}}{\bar{\gamma}_{k+1}}\right\|^2\right], \\
        &\mathcal{L}_{\tilde{f}_{k}^{n+1}}=\mathbb{E}_{(x_{k},x_{N})\sim q_{k,N}^n}\left[\left\|\tilde{f}_{k}^{n+1}(x_{k})-\frac{x_N-x_{k}}{1-\bar{\gamma}_{k}}\right\|^2\right].
    \end{aligned}
    \label{eq:FRDSB_loss}
\end{equation}

\noindent By employing this reparameterization, we are able to generate samples via Equation~\ref{eq:RDSB_1} using $\tilde{\mu}_k^n(x_k,x_N)\approx x_k+\gamma_{k+1}\tilde{f}_{k}^n(x_{k})$ for the forward process, and $\mu_{k+1}^n(x_{k+1},x_0)\approx x_{k+1}+\gamma_{k+1}\tilde{b}_{k+1}^n(x_{k+1})$ for the backward process.

\section{Experiment}

\subsection{Simplified training objectives}

\begin{figure}[t]
    \centering
    \begin{subfigure}[b]{1.0\linewidth}
        \centering
        \includegraphics[width=1.0\linewidth,page=1]{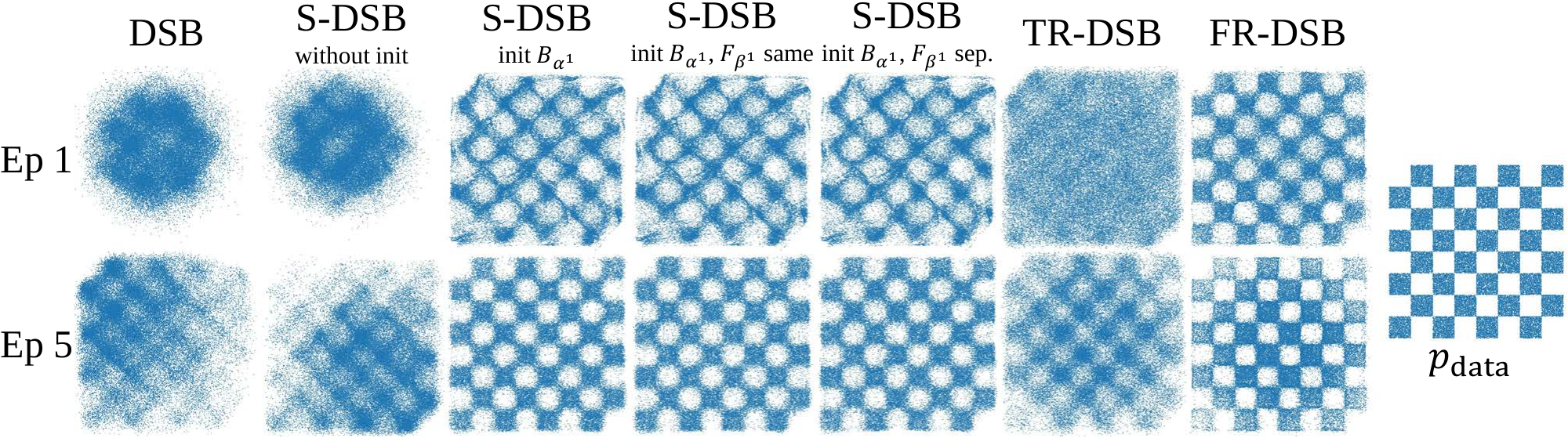}
        \caption{Backward process results}
    \end{subfigure}
    
    \vspace{0.2cm}
    \begin{subfigure}[b]{1.0\linewidth}
        \centering
        \includegraphics[width=1.0\linewidth,page=2]{figure/figure_compare.pdf}
        \caption{Forward process results}
    \end{subfigure}
    \caption{Comparison between DSB, S-DSB with different initialization, and R-DSB on \texttt{checkerboard} $\leftrightarrow$ \texttt{pinwheel}.}
    \vspace{-0.5cm}
    \label{fig:compare_dsb}
\end{figure}

We first compare our proposed Simplified DSB (S-DSB) with the vanilla DSB using 2D synthetic datasets. The data distribution and prior distribution are configured as \texttt{checkerboard} and \texttt{pinwheel}, respectively, as depicted in Figure~\ref{fig:compare_dsb}. For S-DSB, we explore various initialization approaches: (1) \textit{without init}, where $B_{\beta^1}$ and $F_{\alpha^1}$ are trained from scratch; (2) \textit{init $B_{\beta^1}$}, where a pretrained SGM is employed as the initial backward transition $B_{\beta^1}$ and training starts from the subsequent forward epoch; (3) \textit{init $B_{\beta^1}$ and $F_{\alpha^1}$ with same model}, utilizing the same SGM to initialize both $B_{\beta^1}$ and $F_{\alpha^1}$, then start training from the first forward epoch; (4) \textit{init $B_{\beta^1}$ and $F_{\alpha^1}$ separately}, we adopt two distinct SGMs to initialize $B_{\beta^1}$ and $F_{\alpha^1}$, followed by training from the first forward epoch. Furthermore, we also compare these configurations against our Reparameterized DSB (R-DSB).

The results are presented in Figure~\ref{fig:compare_dsb}, showing that both DSB and S-DSB with random initialization exhibit comparable performance, which suggests a theoretical equivalence between them. Notably, S-DSB requires half of network function evaluations (NFEs) for caching training pairs, thereby accelerating the training process. Furthermore, by capitalizing on pretrained SGMs, S-DSB achieves superior results within the same number of training iterations. This finding aligns with our analysis in Section~\ref{sec:convergence_analysis}, which states that the convergence of $q_0^{n}$ and $p_N^{n+1}$ depends on $p_N^{n}$ and $q_0^n$, respectively. Effective initializations that bridge the gap between $p_N^0$ and $p_\textup{prior}$ significantly expedite and stabilize the training of S-DSB. In addition, the results indicate that FR-DSB exhibits robust performance, while the results of TR-DSB are relatively poor. This is analogous to insights from the reparameterization trick in SGMs, suggesting that distinct reparameterization strategies may be more suitable to specific target distributions.

\subsection{Convergence of DSB}

\begin{figure}[t]
\begin{minipage}[c]{0.49\linewidth}
    \includegraphics[width=\linewidth]{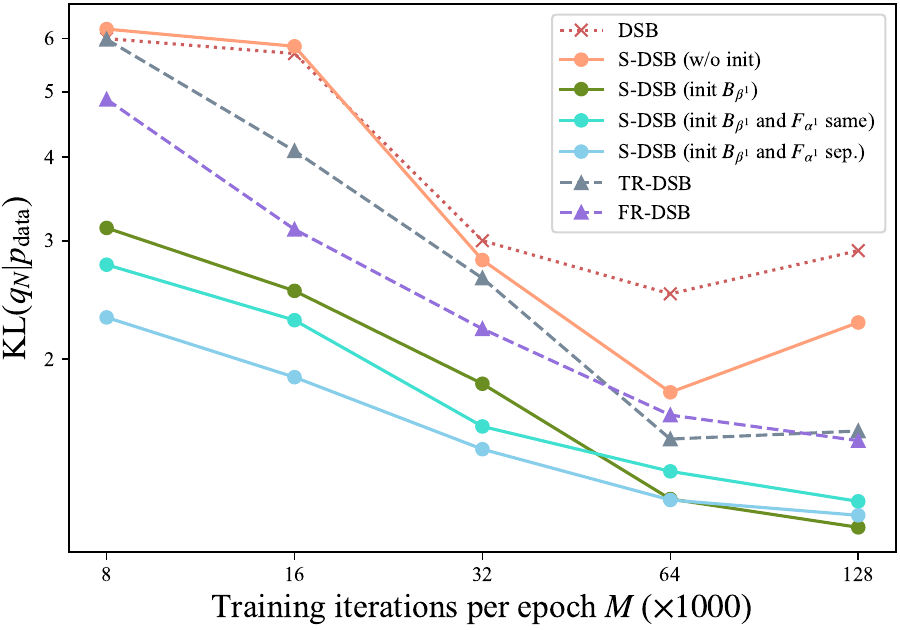}
    \caption{Evolution of performance with the increase of training iterations per epoch.}
    \label{fig:kl_iterations_per_epoch}
\end{minipage}
\hfill
\begin{minipage}[c]{0.49\linewidth}
    \includegraphics[width=\linewidth]{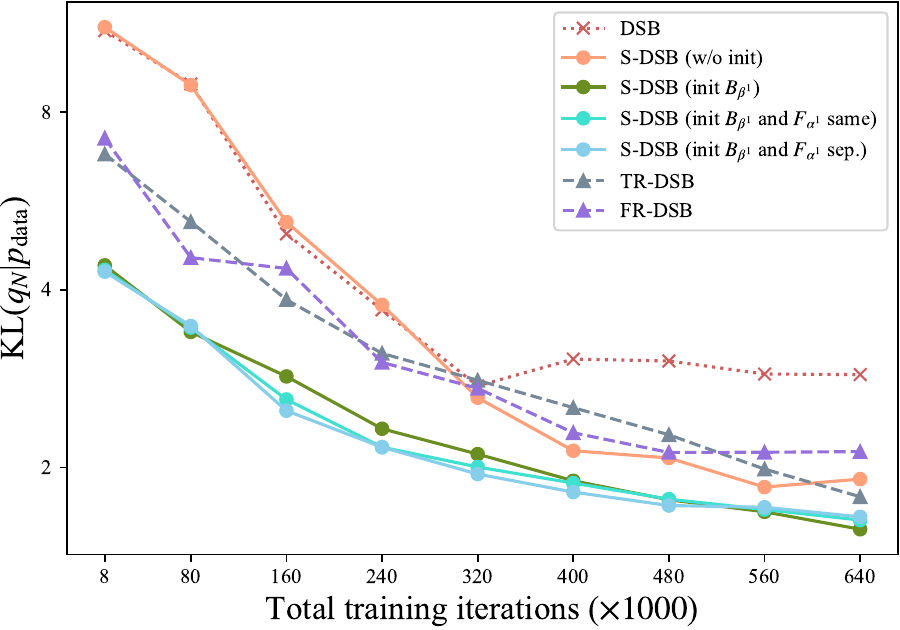}
    \caption{Evolution of performance with the increase of total training iterations.}
    \label{fig:kl_iterations}
\end{minipage}%
\vspace{-0.5cm}
\end{figure}

\begin{figure}[t]
\begin{minipage}[c]{0.65\linewidth}
    \includegraphics[width=1.0\linewidth]{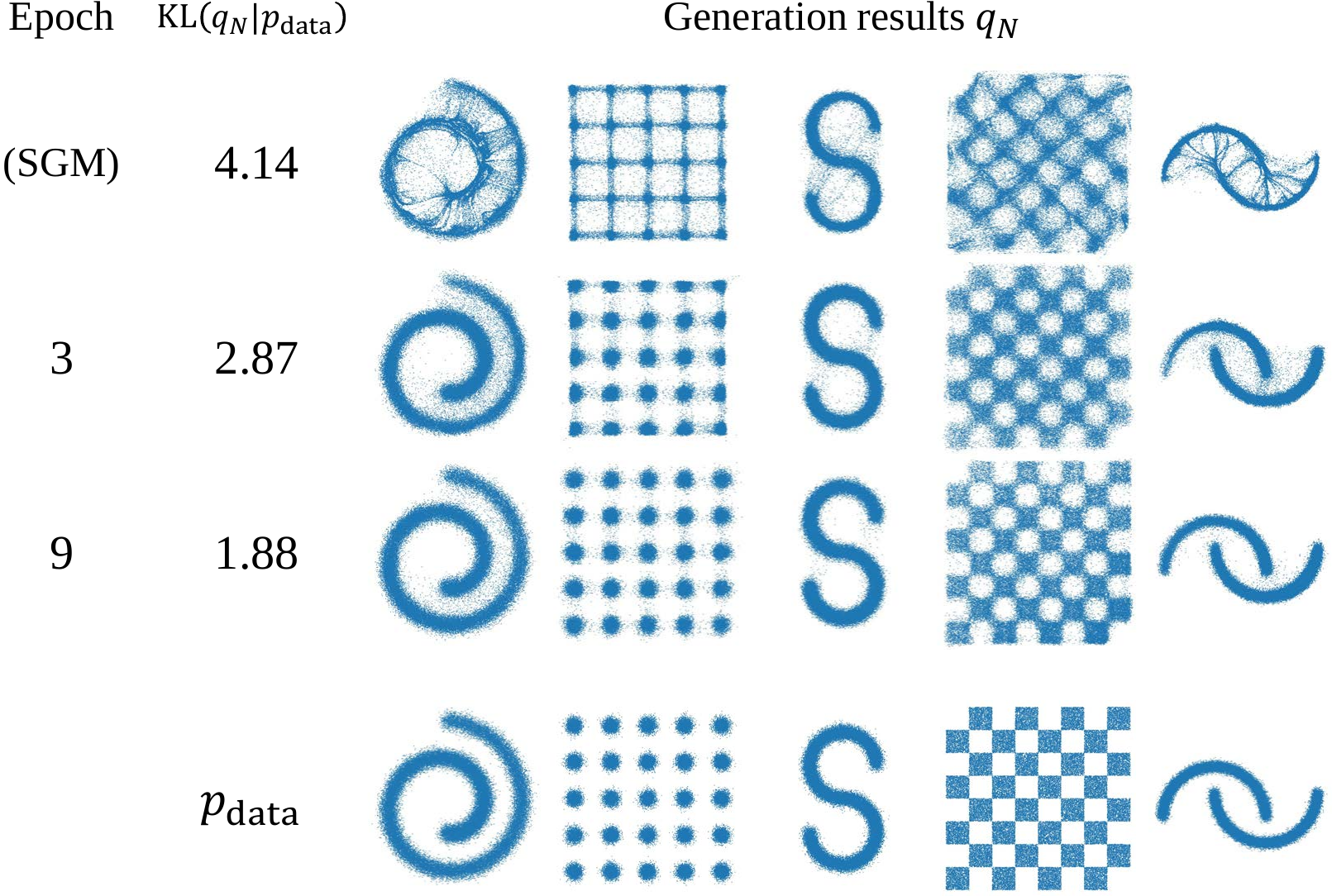}
    \caption{Improving pretrained SGM with S-DSB. First row illustrates the results of converged SGMs. KL divergences $\textup{KL}(q_N|p_\textup{data})$ measure the average generation performance of each rows.}
    \label{fig:improve_diffusion}
\end{minipage}
\hfill
\begin{minipage}[c]{0.3\linewidth}
    \centering
    \begin{subfigure}[b]{1.0\linewidth}
        \centering
        \includegraphics[width=1.0\linewidth]{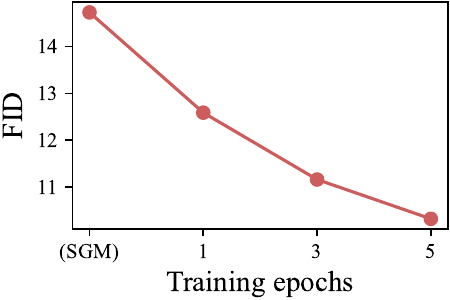}
        \caption{Dog$\rightarrow$Cat}
    \end{subfigure}

    \vspace{0.1cm}
    \begin{subfigure}[b]{1.0\linewidth}
        \centering
        \includegraphics[width=1.0\linewidth]{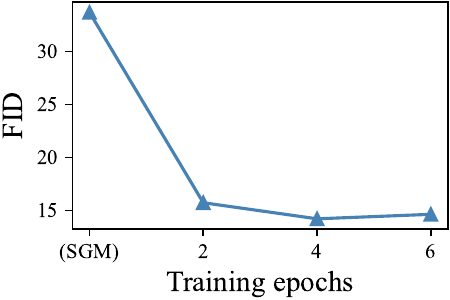}
        \caption{Cat$\rightarrow$Dog}
    \end{subfigure}
    
    \vskip -0.2cm
    \caption{FID comparison of image translation on the AFHQ dataset.}
    \label{fig:afhq_fid}
\end{minipage}%
\vspace{-0.5cm}
\end{figure}

We explore several critical factors that may impact the convergence of DSB, including training iterations per epoch $M$, and the overall training iterations. To gauge convergence, we measured the KL divergence between the output of the backward networks and the real data distribution $p_\textup{data}$. As depicted in Figure~\ref{fig:kl_iterations_per_epoch}, we subject all configurations to training over $10$ epochs. When $M$ is low, the DSB training fails to converge. As $M$ increases, the optimization process gradually approaches Equation~\ref{eq:DSB_definition}, ultimately yielding improved final performance.

Figure~\ref{fig:kl_iterations} elucidates the enhancement of sample quality throughout the training duration. We set the default $M$ as $32K$. We note that the performance of DSB stops improving after 320K iterations. Conversely, both our S-DSB and R-DSB demonstrate consistent improvements. Moreover, even with random initialization, R-DSB surpasses the vanilla DSB in terms of sample quality.

\begin{figure}[t]
    \centering
    \includegraphics[width=1.0\linewidth]{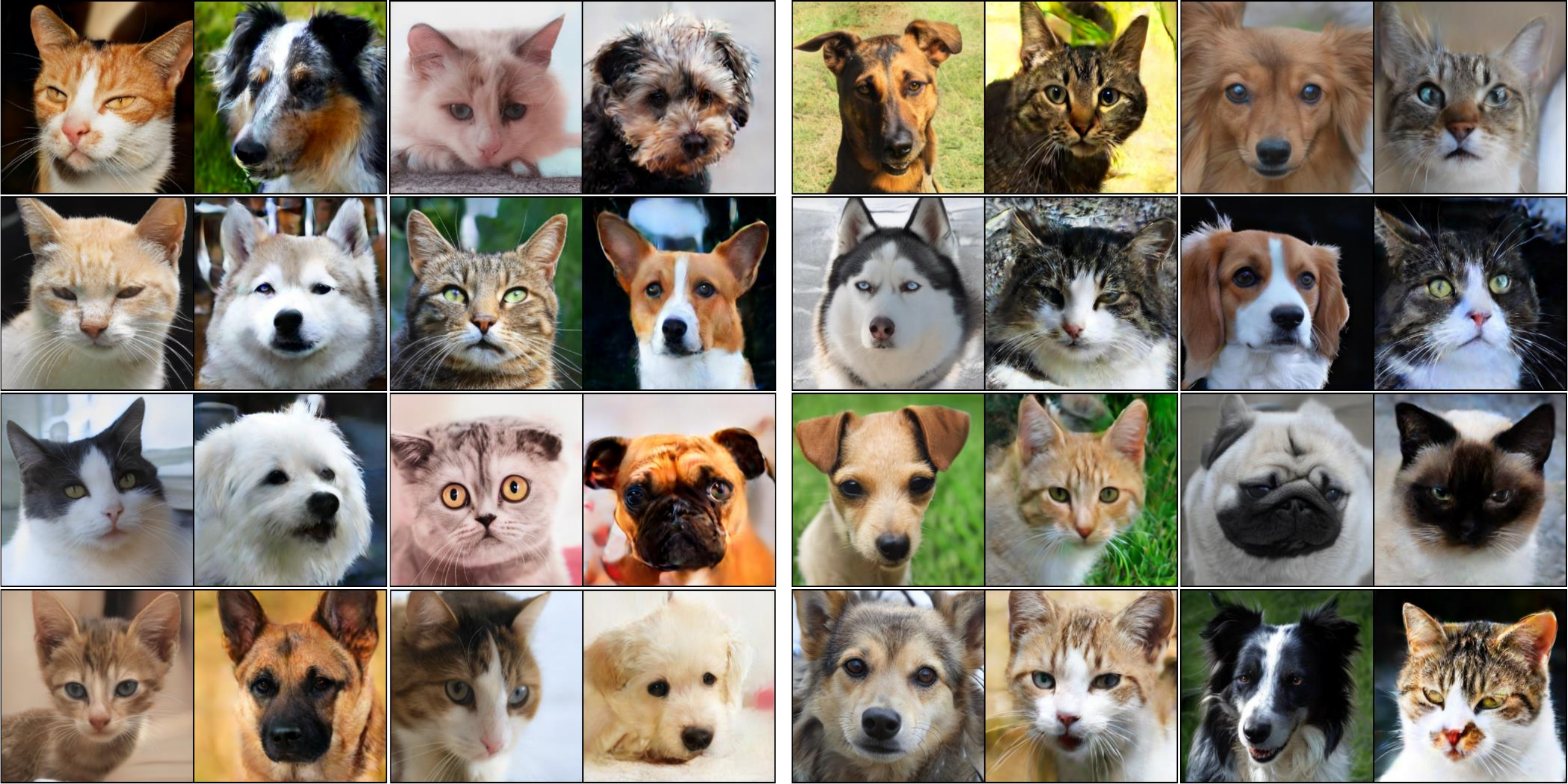}
    \caption{Generated samples on unpaired \texttt{dog} $\leftrightarrow$ \texttt{cat} translation. Left: \texttt{cat} to \texttt{dog}; right: \texttt{dog} to \texttt{cat}. We observe that DSB could preserve pose and texture to a certain extent.}
    \label{fig:afhq-cases}
\end{figure}

\subsection{Improving SGM}

\begin{wrapfigure}[10]{r}{0.5\textwidth}  
    \vspace{-1.5cm}
    \centering  
    \includegraphics[width=1.0\linewidth]{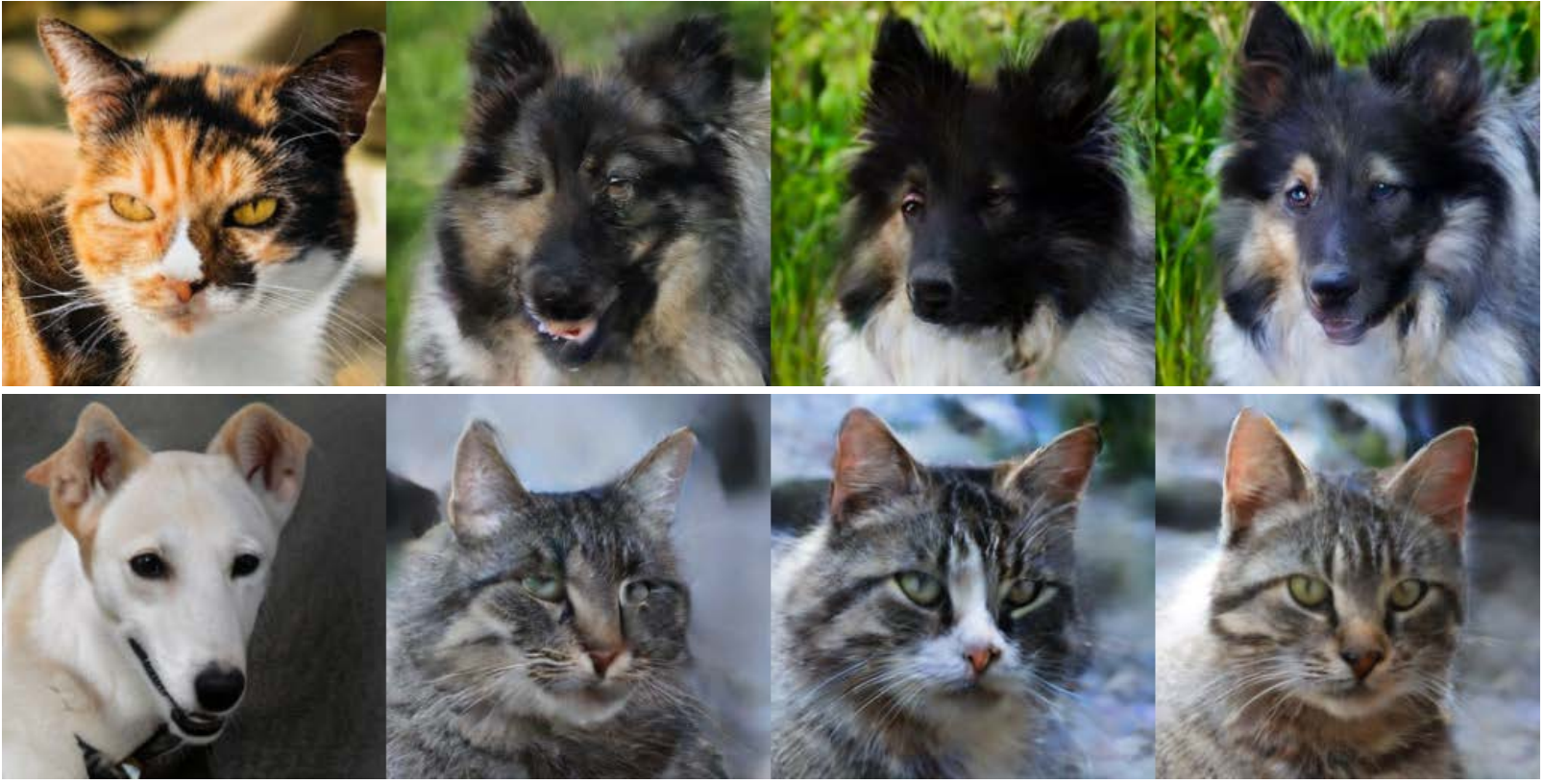}
    \caption{From left to right: input images, results from Flow Matching model, DSB 2-th epoch, and DSB 6-th epoch.}
    \label{fig:afhq-improve}
\end{wrapfigure} 

As discussed in Section~\ref{sec:diffusion_as_dsb}, the simplified training objective in Equation~\ref{eq:DSB_simplified_loss} enables leveraging a pretrained SGM as the initial epoch for S-DSB and further improves it. In Figure~\ref{fig:improve_diffusion}, we first train diffusion models on various data distributions until convergence, and then employ these models as the starting point of S-DSB training. As depicted, the initial performance is suboptimal, with visually discernible artifacts presented in generated samples. However, after few epochs of S-DSB training, the sample quality is greatly improved. With continued training, the results progressively approximate the ground truths.

Additionally, we conducted validation on the AFHQ~\cite{choi2020stargan} dataset to transfer between \texttt{dog} and \texttt{cat} subsets. We utilized a pretrained Flow Matching (FM)~\cite{lipman2022flow} model as our initial model and applied Terminal Reparameterization to train the DSB. Given that FM constrains $p_\textup{prior}$ to be Gaussian, it is understandable that the generated results may not be sufficiently good. As demonstrated in Figure~\ref{fig:afhq_fid} and Figure~\ref{fig:afhq-improve}, our model is capable of progressively generating improved results. This further corroborates the efficacy of DSB in enhancing pretrained SGMs. We provide more translation results in Figure~\ref{fig:afhq-cases}, which illustrate our method is capable of producing results with rich details and realistic textures.

\newpage

\subsection{$\gamma$ term in DSB}

\begin{wrapfigure}[6]{r}{0.27\linewidth}
    \vspace{-1.8cm}
    \centering
    \includegraphics[width=1.0\linewidth]{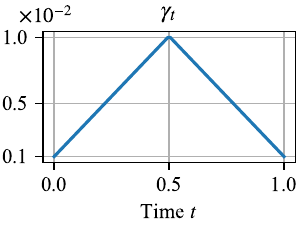}
    \vskip -0.25cm
    \caption{Linear symmetrical schedule for $\gamma_t$.}
    \label{fig:gamma_setting}
\end{wrapfigure}

We adopt a symmetrical schedule for $\gamma_t$ in our experiments, with values that linearly ascend from $10^{-3}$ at both sides to $10^{-2}$ at the center, as depicted in Figure~\ref{fig:gamma_setting}. To examine the influence of $\gamma$, we present visualizations of two transformation trajectories from \texttt{cat} to \texttt{dog} under varying $\gamma_t$ settings in Figure~\ref{fig:afhq-traj}. Our observations indicate that a smaller $\gamma_t$ tends to preserve more information throughout the trajectory, such as pose and color, whereas a larger $\gamma_t$ can potentially lead to high-quality generated results.

\begin{figure}[t]
    \centering
    \includegraphics[width=1.0\linewidth]{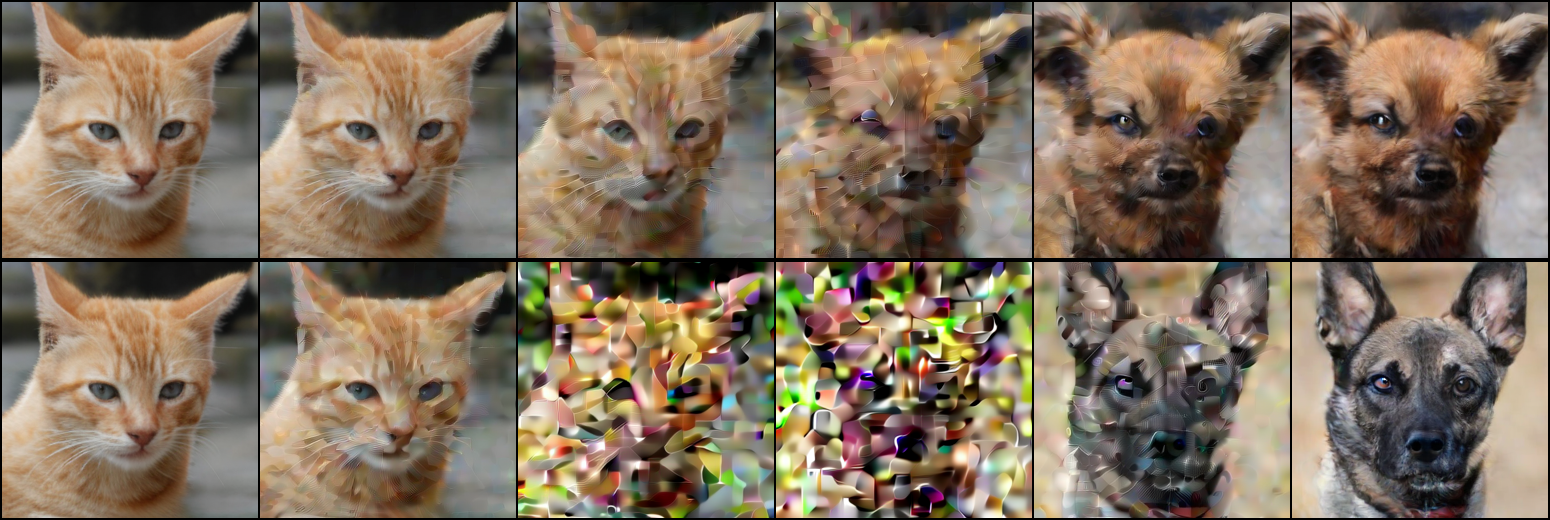}
    \caption{Generation trajectories with different $\gamma_t$. Taking the same \texttt{cat} image as input, the trajectory with smaller $\gamma_t$ can retain more details, such as pose and color. Above: $\gamma_t \in \text{linear}(10^{-4}, 10^{-3})$. Below: $\gamma_t \in \text{linear}(10^{-3}, 10^{-2})$.}
    \label{fig:afhq-traj}
    \vspace{-0.5cm}
\end{figure}

\section{Conclusion and Future Prospects}
This paper proposes Simplified Diffusion Schrödinger Bridge (S-DSB), which unifies the training objectives of Diffusion Schrödinger Bridge (DSB) and Score-based generative models (SGMs). By leveraging SGMs as initialization, S-DSB achieves faster convergence and improved performance. The iterative refinement process inherent in DSB further enhances the outcomes of SGM. Our investigation delves deeply into the practical training of DSB, identifying critical factors that influence its success. Moreover, we introduce reparameterization techniques tailored for DSB that yield superior performance in real-world applications.

The authors firmly believe that approaches based on the Schrödinger Bridge problem represent the forward path for future generative models, as previous studies~\cite{de2021diffusion,shi2024diffusion} have unveiled its close connection with the optimal transport theory. Consequently, DSB has the theoretical potential to possess stronger generation capabilities and a deeper understanding of semantics. However, the authors acknowledge that the current DSB may still lag behind SGMs due to the complexity of its training process or a lack of effective reparameterization strategy to alleviate difficulties encountered during network training. This paper makes exploratory strides toward this goal and hopes to inspire future works.

%
%
\bibliographystyle{splncs04}
\bibliography{main}

\begin{thebibliography}{10}
\providecommand{\url}[1]{\texttt{#1}}
\providecommand{\urlprefix}{URL }
\providecommand{\doi}[1]{https://doi.org/#1}

\bibitem{sora2024}
Video generation models as world simulators. \url{https://openai.com/research/video-generation-models-as-world-simulators}, accessed: 2024-03-01

\bibitem{anderson1982reverse}
Anderson, B.D.: Reverse-time diffusion equation models. Stochastic Processes and their Applications  \textbf{12}(3),  313--326 (1982)

\bibitem{balaji2022ediffi}
Balaji, Y., Nah, S., Huang, X., Vahdat, A., Song, J., Kreis, K., Aittala, M., Aila, T., Laine, S., Catanzaro, B., et~al.: ediffi: Text-to-image diffusion models with an ensemble of expert denoisers. arXiv preprint arXiv:2211.01324  (2022)

\bibitem{bernton2019schr}
Bernton, E., Heng, J., Doucet, A., Jacob, P.E.: Schr{\"o}dinger bridge samplers. arXiv preprint arXiv:1912.13170  (2019)

\bibitem{blattmann2023stable}
Blattmann, A., Dockhorn, T., Kulal, S., Mendelevitch, D., Kilian, M., Lorenz, D., Levi, Y., English, Z., Voleti, V., Letts, A., et~al.: Stable video diffusion: Scaling latent video diffusion models to large datasets. arXiv preprint arXiv:2311.15127  (2023)

\bibitem{caluya2021wasserstein}
Caluya, K.F., Halder, A.: Wasserstein proximal algorithms for the schr{\"o}dinger bridge problem: Density control with nonlinear drift. IEEE Transactions on Automatic Control  \textbf{67}(3),  1163--1178 (2021)

\bibitem{cattiaux2021time}
Cattiaux, P., Conforti, G., Gentil, I., L{\'e}onard, C.: Time reversal of diffusion processes under a finite entropy condition. arXiv preprint arXiv:2104.07708  (2021)

\bibitem{chen2023importance}
Chen, T.: On the importance of noise scheduling for diffusion models. arXiv preprint arXiv:2301.10972  (2023)

\bibitem{chen2016entropic}
Chen, Y., Georgiou, T., Pavon, M.: Entropic and displacement interpolation: a computational approach using the hilbert metric. SIAM Journal on Applied Mathematics  \textbf{76}(6),  2375--2396 (2016)

\bibitem{chen2020optimal}
Chen, Y., Georgiou, T.T., Pavon, M.: Optimal transport in systems and control. Annual Review of Control, Robotics, and Autonomous Systems  \textbf{4} (2021)

\bibitem{chen2023schrodinger}
Chen, Z., He, G., Zheng, K., Tan, X., Zhu, J.: Schrodinger bridges beat diffusion models on text-to-speech synthesis. arXiv preprint arXiv:2312.03491  (2023)

\bibitem{choi2020stargan}
Choi, Y., Uh, Y., Yoo, J., Ha, J.W.: Stargan v2: Diverse image synthesis for multiple domains. In: Proceedings of the IEEE/CVF conference on computer vision and pattern recognition. pp. 8188--8197 (2020)

\bibitem{de2021diffusion}
De~Bortoli, V., Thornton, J., Heng, J., Doucet, A.: Diffusion schr{\"o}dinger bridge with applications to score-based generative modeling. Advances in Neural Information Processing Systems  \textbf{34},  17695--17709 (2021)

\bibitem{dhariwal2021diffusion}
Dhariwal, P., Nichol, A.: Diffusion models beat gans on image synthesis. Advances in neural information processing systems  \textbf{34},  8780--8794 (2021)

\bibitem{finlay2020learning}
Finlay, C., Gerolin, A., Oberman, A.M., Pooladian, A.A.: Learning normalizing flows from entropy-kantorovich potentials. arXiv preprint arXiv:2006.06033  (2020)

\bibitem{follmer1985entropy}
F{\"o}llmer, H.: An entropy approach to the time reversal of diffusion processes. In: Stochastic Differential Systems: Filtering and Control, pp. 156--163. Springer (1985)

\bibitem{follmer1988random}
F{\"o}llmer, H.: Random fields and diffusion processes. In: {\'E}cole d'{\'E}t{\'e} de Probabilit{\'e}s de Saint-Flour XV--XVII, 1985--87, pp. 101--203. Springer (1988)

\bibitem{fortet1940resolution}
Fortet, R.: R{\'e}solution d’un syst{\`e}me d’{\'e}quations de {M}. {S}chr{\"o}dinger. Journal de Math{\'e}matiques Pures et Appliqu{\'e}s  \textbf{1},  83--105 (1940)

\bibitem{gu2022f}
Gu, J., Zhai, S., Zhang, Y., Bautista, M.A., Susskind, J.: f-dm: A multi-stage diffusion model via progressive signal transformation. International Conference on Learning Representations  (2021)

\bibitem{gu2022vector}
Gu, S., Chen, D., Bao, J., Wen, F., Zhang, B., Chen, D., Yuan, L., Guo, B.: Vector quantized diffusion model for text-to-image synthesis. In: Proceedings of the IEEE/CVF Conference on Computer Vision and Pattern Recognition. pp. 10696--10706 (2022)

\bibitem{hang2023efficient}
Hang, T., Gu, S., Li, C., Bao, J., Chen, D., Hu, H., Geng, X., Guo, B.: Efficient diffusion training via min-snr weighting strategy. arXiv preprint arXiv:2303.09556  (2023)

\bibitem{haussmann1986time}
Haussmann, U.G., Pardoux, E.: Time reversal of diffusions. The Annals of Probability  \textbf{14}(4),  1188--1205 (1986)

\bibitem{ho2022imagen}
Ho, J., Chan, W., Saharia, C., Whang, J., Gao, R., Gritsenko, A., Kingma, D.P., Poole, B., Norouzi, M., Fleet, D.J., et~al.: Imagen video: High definition video generation with diffusion models. arXiv preprint arXiv:2210.02303  (2022)

\bibitem{ho2020denoising}
Ho, J., Jain, A., Abbeel, P.: Denoising diffusion probabilistic models. Advances in neural information processing systems  \textbf{33},  6840--6851 (2020)

\bibitem{hyvarinen2005estimation}
Hyv{\"a}rinen, A., Dayan, P.: Estimation of non-normalized statistical models by score matching. Journal of Machine Learning Research  \textbf{6}(4) (2005)

\bibitem{jun2023shap}
Jun, H., Nichol, A.: Shap-e: Generating conditional 3d implicit functions. arXiv preprint arXiv:2305.02463  (2023)

\bibitem{karras2022elucidating}
Karras, T., Aittala, M., Aila, T., Laine, S.: Elucidating the design space of diffusion-based generative models. Advances in Neural Information Processing Systems  \textbf{35},  26565--26577 (2022)

\bibitem{kullback1968probability}
Kullback, S.: Probability densities with given marginals. The Annals of Mathematical Statistics  \textbf{39}(4),  1236--1243 (1968)

\bibitem{leonard2014survey}
L{\'e}onard, C.: A survey of the {S}chr{\"o}dinger problem and some of its connections with optimal transport. Discrete \& Continuous Dynamical Systems-A  \textbf{34}(4),  1533--1574 (2014)

\bibitem{lin2024common}
Lin, S., Liu, B., Li, J., Yang, X.: Common diffusion noise schedules and sample steps are flawed. In: Proceedings of the IEEE/CVF Winter Conference on Applications of Computer Vision. pp. 5404--5411 (2024)

\bibitem{lipman2022flow}
Lipman, Y., Chen, R.T., Ben-Hamu, H., Nickel, M., Le, M.: Flow matching for generative modeling. International Conference on Learning Representations  (2023)

\bibitem{liu20232}
Liu, G.H., Vahdat, A., Huang, D.A., Theodorou, E.A., Nie, W., Anandkumar, A.: I $^2$ sb: Image-to-image schr{\"o}dinger bridge. International Conference on Machine Learning  (2023)

\bibitem{liu2022flow}
Liu, X., Gong, C., Liu, Q.: Flow straight and fast: Learning to generate and transfer data with rectified flow. International Conference on Learning Representations  (2023)

\bibitem{liu2015faceattributes}
Liu, Z., Luo, P., Wang, X., Tang, X.: Deep learning face attributes in the wild. In: Proceedings of International Conference on Computer Vision (ICCV) (December 2015)

\bibitem{nelson1967dynamical}
Nelson, E.: Dynamical theories of brownian motion  (1967)

\bibitem{nichol2022point}
Nichol, A., Jun, H., Dhariwal, P., Mishkin, P., Chen, M.: Point-e: A system for generating 3d point clouds from complex prompts. arXiv preprint arXiv:2212.08751  (2022)

\bibitem{pavon2021data}
Pavon, M., Trigila, G., Tabak, E.G.: The data-driven schr{\"o}dinger bridge. Communications on Pure and Applied Mathematics  \textbf{74}(7),  1545--1573 (2021)

\bibitem{ramesh2022hierarchical}
Ramesh, A., Dhariwal, P., Nichol, A., Chu, C., Chen, M.: Hierarchical text-conditional image generation with clip latents. arXiv preprint arXiv:2204.06125  \textbf{1}(2), ~3 (2022)

\bibitem{ramesh2021zero}
Ramesh, A., Pavlov, M., Goh, G., Gray, S., Voss, C., Radford, A., Chen, M., Sutskever, I.: Zero-shot text-to-image generation. In: International Conference on Machine Learning. pp. 8821--8831. PMLR (2021)

\bibitem{rombach2022high}
Rombach, R., Blattmann, A., Lorenz, D., Esser, P., Ommer, B.: High-resolution image synthesis with latent diffusion models. In: Proceedings of the IEEE/CVF conference on computer vision and pattern recognition. pp. 10684--10695 (2022)

\bibitem{ruschendorf1993note}
R{\"u}schendorf, L., Thomsen, W.: Note on the schr{\"o}dinger equation and i-projections. Statistics \& probability letters  \textbf{17}(5),  369--375 (1993)

\bibitem{ruschendorf1995convergence}
Ruschendorf, L., et~al.: Convergence of the iterative proportional fitting procedure. The Annals of Statistics  \textbf{23}(4),  1160--1174 (1995)

\bibitem{salimans2022progressive}
Salimans, T., Ho, J.: Progressive distillation for fast sampling of diffusion models. International Conference on Learning Representations  (2022)

\bibitem{schrodinger1932theorie}
Schr{\"o}dinger, E.: Sur la th{\'e}orie relativiste de l'{\'e}lectron et l'interpr{\'e}tation de la m{\'e}canique quantique. Annales de l'Institut Henri Poincar{\'e}  \textbf{2}(4),  269--310 (1932)

\bibitem{shi2024diffusion}
Shi, Y., De~Bortoli, V., Campbell, A., Doucet, A.: Diffusion schr{\"o}dinger bridge matching. Advances in Neural Information Processing Systems  \textbf{36} (2024)

\bibitem{sohl2015deep}
Sohl-Dickstein, J., Weiss, E., Maheswaranathan, N., Ganguli, S.: Deep unsupervised learning using nonequilibrium thermodynamics. In: International conference on machine learning. pp. 2256--2265. PMLR (2015)

\bibitem{song2020denoising}
Song, J., Meng, C., Ermon, S.: Denoising diffusion implicit models. International Conference on Learning Representations  (2021)

\bibitem{song2019generative}
Song, Y., Ermon, S.: Generative modeling by estimating gradients of the data distribution. Advances in neural information processing systems  \textbf{32} (2019)

\bibitem{song2020score}
Song, Y., Sohl-Dickstein, J., Kingma, D.P., Kumar, A., Ermon, S., Poole, B.: Score-based generative modeling through stochastic differential equations. International Conference on Learning Representations  (2021)

\bibitem{tang2023volumediffusion}
Tang, Z., Gu, S., Wang, C., Zhang, T., Bao, J., Chen, D., Guo, B.: Volumediffusion: Flexible text-to-3d generation with efficient volumetric encoder. arXiv preprint arXiv:2312.11459  (2023)

\bibitem{tong2023improving}
Tong, A., Malkin, N., Huguet, G., Zhang, Y., Rector-Brooks, J., Fatras, K., Wolf, G., Bengio, Y.: Improving and generalizing flow-based generative models with minibatch optimal transport. In: ICML Workshop on New Frontiers in Learning, Control, and Dynamical Systems (2023)

\bibitem{vargas2021solving}
Vargas, F., Thodoroff, P., Lamacraft, A., Lawrence, N.: Solving schr{\"o}dinger bridges via maximum likelihood. Entropy  \textbf{23}(9), ~1134 (2021)

\bibitem{vincent2011connection}
Vincent, P.: A connection between score matching and denoising autoencoders. Neural Computation  \textbf{23}(7),  1661--1674 (2011)

\bibitem{zhu2017unpaired}
Zhu, J.Y., Park, T., Isola, P., Efros, A.A.: Unpaired image-to-image translation using cycle-consistent adversarial networks. In: Proceedings of the IEEE international conference on computer vision. pp. 2223--2232 (2017)

\end{thebibliography}

\newpage
\appendix

\section{Proof of Proposition~\ref{prop:SDSB}}

\begin{proof}
    By asssumption, we have
    \begin{equation}
        \begin{aligned}
            &\ \left(x_{k+1}+F_k^n(x_k)-F_k^n(x_{k+1})\right)-x_k \\
            =&\ \left(x_{k+1}+x_{k}+\gamma_{k+1}f_k^n(x_{k})-\left(x_{k+1}+\gamma_{k+1}f_k^n(x_{k+1})\right)\right)-x_k \\
            =&\ \left(x_{k}+\gamma_{k+1}f_k^n(x_{k})-\gamma_{k+1}f_k^n(x_{k+1})\right)-x_k \\
            =&\ \gamma_{k+1}f_k^n(x_{k})-\gamma_{k+1}f_k^n(x_{k+1}) \\
            =&\ f_k^n(x_{k})-f_k^n(x_{k+1}) \\
            \approx&\ 0,
        \end{aligned}
    \end{equation}
     and thus $\mathcal{L}_{B_{k+1}^n}^{'}\approx\mathcal{L}_{B_{k+1}^n}$. $\mathcal{L}_{F_k^{n+1}}^{'}\approx\mathcal{L}_{F_k^{n+1}}$ can also be proved likewise.
\end{proof}

\section{Proof of Proposition~\ref{prop:RDSB_prop}}

\subsection{Continuous-time DSB}

To prove Proposition~\ref{prop:RDSB_prop}, we first introduce continuous-time DSB. Given a reference measure $\mathbb{P}\in\mathscr{P}(\mathcal{C})$, the continuous-time SB problem aims to find

\begin{equation}
    \begin{aligned}
        \mathrm{\Pi}^*=\arg\min\{\textup{KL}(\mathrm{\Pi}|\mathbb{P}):\mathrm{\Pi}\in\mathscr{P}(\mathcal{C}),\mathrm{\Pi}_0=p_\textup{data},\mathrm{\Pi}_T=p_\textup{prior}\}.
    \end{aligned}
    \label{eq:SB_definition_continuous}
\end{equation}

Similar to Equation~\ref{eq:IPF_definition}, continuous-time IPF iterations $(\mathrm{\Pi}^n)_{n\in\mathbb{N}}$ where $\mathrm{\Pi}^0=\mathbb{P}$ for any $n\in\mathbb{N}$ can be defined in continuous time as

\begin{equation}
    \begin{aligned}
        \mathrm{\Pi}^{2n+1}&=\arg\min\{\textup{KL}(\mathrm{\Pi}|\mathrm{\Pi}^{2n}):\mathrm{\Pi}\in\mathscr{P}(\mathcal{C}),\mathrm{\Pi}_T=p_\textup{prior}\},\\
        \mathrm{\Pi}^{2n+2}&=\arg\min\{\textup{KL}(\mathrm{\Pi}|\mathrm{\Pi}^{2n+1}):\mathrm{\Pi}\in\mathscr{P}(\mathcal{C}),\mathrm{\Pi}_0=p_\textup{data}\}.
    \end{aligned}
    \label{eq:IPF_definition_continuous}
\end{equation}

DSB partitions the continuous time horizon $[0,T]$ into $N+1$ discrete timesteps $\{\bar{\gamma}_{0},\bar{\gamma}_{1},\ldots,\bar{\gamma}_{N}\}$ where $\bar{\gamma}_k=\sum_{n=1}^k\gamma_{n}$, $\bar{\gamma}_0=0$ and $\bar{\gamma}_N=T$. The forward process is $p_{\bar{\gamma}_{k+1}|\bar{\gamma}_{k}}^n=\mathcal{N}\left(x_{\bar{\gamma}_{k+1}};x_{\bar{\gamma}_{k}}+\gamma_{k+1}f_{\bar{\gamma}_{k}}^n(x_{\bar{\gamma}_{k}}),2\gamma_{k+1}\mathbf{I}\right)$ and the backward process is $q_{\bar{\gamma}_{k}|\bar{\gamma}_{k+1}}^n=\mathcal{N}\left(x_{\bar{\gamma}_{k}};x_{\bar{\gamma}_{k+1}}+\gamma_{k+1}b_{\bar{\gamma}_{k+1}}^n(x_{\bar{\gamma}_{k+1}}),2\gamma_{k+1}\mathbf{I}\right)$, where $f_{\bar{\gamma}_{k}}^n(x_{\bar{\gamma}_{k}})$ and $b_{\bar{\gamma}_{k+1}}^n(x_{\bar{\gamma}_{k+1}})$ are drift terms, and $p^n$ and $q^n$ are discretizations of $\mathrm{\Pi}^{2n}$ and $\mathrm{\Pi}^{2n+1}$, respectively. DSB uses two neural networks to approximate $B_{\beta^n}(\bar{\gamma}_k,x)\approx B_{\bar{\gamma}_k}^n(x)=x+\bar{\gamma}_kb_{\bar{\gamma}_k}^n(x)$ and $F_{\alpha^n}(\bar{\gamma}_k,x)\approx F_{\bar{\gamma}_k}^n(x)=x+\bar{\gamma}_{k+1}f_{\bar{\gamma}_{k}}^n(x)$.

\subsection{Posterior estimation}

First, we introduce some notations in continuous-time SDEs. We rewrite Equation~\ref{eq:SB_SDE} as

\begin{subequations}
    \begin{align}
        \mathrm{d}\mathbf{X}_t&=\left(f(\mathbf{X}_t,t)+g^2(t)\nabla\log\mathbf{\Psi}(\mathbf{X}_t,t)\right)\mathrm{d}t+g(t)\mathrm{d}\mathbf{W}_t, X_0\sim p_\textup{data} \nonumber \\
        &:=h(\mathbf{X}_t,t)\mathrm{d}t+g(t)\mathrm{d}\mathbf{W}_t, X_0\sim p_\textup{data}, \label{eq:SB_SDE_rewrite_1} \\
        \mathrm{d}\mathbf{X}_t&=\left(f(\mathbf{X}_t,t)-g^2(t)\nabla\log\mathbf{\hat{\Psi}}(\mathbf{X}_t,t)\right)\mathrm{d}t+g(t)\mathrm{d}\mathbf{\bar{W}}_t, X_0\sim p_\textup{data} \nonumber \\
        &:=l(\mathbf{X}_t,t)\mathrm{d}t+g(t)\mathrm{d}\mathbf{\bar{W}}_t, X_1\sim p_\textup{prior}, \label{eq:SB_SDE_rewrite_2}
    \end{align}
    \label{eq:SB_SDE_rewrite}
\end{subequations}

\noindent where $h(\mathbf{X}_t,t)$ and $l(\mathbf{X}_t,t)$ are drift terms. Given the forward trajectory $x_{0:N}\sim p^n(x_{0:N})$ and the backward trajectory $x'_{0:N}\sim q^n(x_{0:N})$, $\{x_{\bar{\gamma}_i}:=x_{i}\}_{i=0}^{N}$ and $\{x'_{\bar{\gamma}_i}:=x'_{i}\}_{i=0}^{N}$ denote the trajectories of Equation~\ref{eq:SB_SDE_rewrite_1} and Equation~\ref{eq:SB_SDE_rewrite_2} at time $\{\bar{\gamma}_0,\bar{\gamma}_1,\ldots,\bar{\gamma}_N\}$, respectively. For simplicity, we denote the integration $\int_a^bh(\mathbf{X}_t,t)$ and $\int_a^bl(\mathbf{X}_t,t)$ as $\tilde{h}(a,b)$ and $\tilde{l}(a,b)$, respectively.

Here we give a full version of Proposition~\ref{prop:RDSB_prop}:

\begin{proposition}
    Assume $\bar{\gamma}_{N}=1$. Given the forward and backward trajectories $x_{0:N}\sim p^n(x_{0:N})=p_0^n(x_0)\prod_{k=0}^{N-1}p_{k+1|k}^n(x_{k+1}|x_k)$ and $x_{0:N}'\sim q^n(x_{0:N})=q_N^n(x_N)\prod_{k=0}^{N-1}p_{k|k+1}^n(x_k|x_{k+1})$. Assume $f_t(\mathbf{X}_t)=-\alpha\mathbf{X}_t$, $g(t)=\sqrt{2}$,
    \begin{equation}
        \begin{aligned}
            \frac{1}{\bar{\gamma}_{k}}\tilde{h}(0,\bar{\gamma}_k)\approx\frac{1}{\bar{\gamma}_{k+1}}\tilde{h}(0,\bar{\gamma}_{k+1}), \frac{1}{1-\bar{\gamma}_{k}}\tilde{l}(\bar{\gamma}_{k},1)\approx\frac{1}{1-\bar{\gamma}_{k+1}}\tilde{l}(\bar{\gamma}_{k+1},1).
        \end{aligned}
        \label{eq:RDSB_assumption}
    \end{equation}
    We have
    \begin{subequations}
        \begin{align}
            & q_{k|k+1}^n(x_{k}|x_{k+1})\approx p_{k|k+1,0}^n(x_{k}|x_{k+1},x_0) = \mathcal{N}(x_{k};\mu_{k+1}^n(x_{k+1},x_0),\sigma_{k+1}\mathbf{I}), \label{eq:RDSB_prop_1_a} \\
            & p_{k+1|k}^{n+1}(x_{k+1}|x_{k})\approx q_{k+1|k,N}^n(x_{k+1}'|x_k',x_N') = \mathcal{N}(x_{k+1}';\tilde{\mu}_k^n(x_k',x_N'),\tilde{\sigma}_{k+1}\mathbf{I}), \label{eq:RDSB_prop_1_b}
        \end{align}
        \label{eq:RDSB_prop_1}
    \end{subequations}
    \begin{subequations}
        \begin{align}
            & \mu_{k+1}^n(x_{k+1},x_0)\approx x_{k+1}+\frac{\gamma_{k+1}}{\bar{\gamma}_{k+1}}(x_0-x_{k+1}), \sigma_{k+1}=\frac{2\gamma_{k+1}\bar{\gamma}_{k}}{\bar{\gamma}_{k+1}}, \label{eq:RDSB_prop_2_a} \\
            & \tilde{\mu}^n_k(x_k',x_N')\approx x_k'+\frac{\gamma_{k+1}}{1-\bar{\gamma}_{k}}(x_N'-x_k'), \tilde{\sigma}_{k+1}=\frac{2\gamma_{k+1}(1-\bar{\gamma}_{k+1})}{1-\bar{\gamma}_k}. \label{eq:RDSB_prop_2_b}
        \end{align}
        \label{eq:RDSB_prop_2}
    \end{subequations}
    \label{prop:RDSB_prop_full}
\end{proposition}

It is worth noticing that $f_t(\mathbf{X}_t)=-\alpha\mathbf{X}_t$ and $g(t)=\sqrt{2}$ are assumptions of the original DSB~\cite{de2021diffusion}. We prove Proposition~\ref{prop:RDSB_prop_full} with the following lemmas.

\begin{lemma}
    Given the forward and backward trajectories $x_{0:N}\sim p^n(x_{0:N})=p_0^n(x_0)\prod_{k=0}^{N-1}p_{k+1|k}^n(x_{k+1}|x_k)$ and $x_{0:N}'\sim q^n(x_{0:N})=q_N^n(x_N)\prod_{k=0}^{N-1}p_{k|k+1}^n(x_k|x_{k+1})$. Suppose the IPF in Equation~\ref{eq:IPF_definition_continuous} achieve the optimal state. Then we have
    \begin{subequations}
        \begin{align}
            & q_{k|k+1}^n(x_{k}|x_{k+1})\approx p_{k|k+1,0}^n(x_{k}|x_{k+1},x_0), \label{eq:RDSB_prop_1_a_lemma_0} \\
            & p_{k+1|k}^{n+1}(x_{k+1}|x_{k})\approx q_{k+1|k,N}^n(x_{k+1}'|x_k',x_N'). \label{eq:RDSB_prop_1_b_lemma_0}
        \end{align}
        \label{eq:RDSB_prop_1_lemma_0}
    \end{subequations}
    \label{lem:RDSB_prop_1_lemma_0}
\end{lemma}

\begin{proof}
    We first prove Equation~\ref{eq:RDSB_prop_1_a_lemma_0}. In $(2n+1)$-th iteration, IPF aims to solve
    \begin{equation}
        \label{eq:IPF_custom_decomposition}
        \begin{aligned}
        \textup{KL}\left(q^n|p^{n}\right)&=\mathbb{E}_{q^n}\left(\log\frac{q^n(x_N)\prod_{k=0}^{N-1}q^n(x_{k}|x_{k+1})}{p^n(x_0)p^n(x_N|x_0)\prod_{k=1}^{N-1}p^n(x_{k}|x_{k+1},x_0)}\right)\\
            &=\mathbb{E}_{q^n}\bigg(\underbrace{\log\frac{q^n(x_0|x_1)}{p^n(x_0)}}_\textup{reconstruction}+\underbrace{\log\frac{q^n(x_N)}{p^n(x_N|x_0)}}_\textup{prior matching}+\sum_{k=1}^{N-1}\underbrace{\log\frac{q^n(x_{k}|x_{k+1})}{p^n(x_{k}|x_{k+1},x_0)}}_\textup{denoising matching}\bigg).
        \end{aligned}
    \end{equation}
    On the minimization of Equation~\ref{eq:IPF_custom_decomposition}, we have the denoising matching terms to be zero and thus Equation~\ref{eq:RDSB_prop_1_a_lemma_0}. We can prove Equation~\ref{eq:RDSB_prop_1_b_lemma_0} likewise.
\end{proof}

\begin{lemma}
    Assume $\bar{\gamma}_{N}=1$. Given the forward trajectory $x_{0:N}\sim p^n(x_{0:N})=p_0^n(x_0)\prod_{k=0}^{N-1}p_{k+1|k}^n(x_{k+1}|x_k)$ and the and backward trajectory $x_{0:N}'\sim q^n(x_{0:N})=q_N^n(x_N)\prod_{k=0}^{N-1}p_{k|k+1}^n(x_k|x_{k+1})$. Assume $f_t(\mathbf{X}_t)=-\alpha\mathbf{X}_t$ and $g(t)=\sqrt{2}$. We have
    \begin{subequations}
        \begin{align}
            & p_{k|k+1,0}^n(x_{k}|x_{k+1},x_0) = \mathcal{N}(x_{k};\mu_{k+1}^n(x_{k+1},x_0),\sigma_{k+1}\mathbf{I}), \label{eq:RDSB_prop_1_a_lemma_1} \\
            & q_{k+1|k,N}^n(x_{k+1}'|x_k',x_N') = \mathcal{N}(x_{k+1}';\tilde{\mu}_k^n(x_k',x_N'),\tilde{\sigma}_{k+1}\mathbf{I}), \label{eq:RDSB_prop_1_b_lemma_1}
        \end{align}
        \label{eq:RDSB_prop_1_lemma_1}
    \end{subequations}
    where
    \begin{subequations}
        \begin{align}
            & \mu_{k+1}^n(x_{k+1},x_0)=x_{k+1}+\frac{\gamma_{k+1}}{\bar{\gamma}_{k+1}}(x_0-x_{k+1})-\tilde{h}(\bar{\gamma}_k,\bar{\gamma}_{k+1})+\frac{\gamma_{k+1}}{\bar{\gamma}_{k+1}}\tilde{h}(0,\bar{\gamma}_{k+1}), \nonumber \\
            & \sigma_{k+1}=\frac{2\gamma_{k+1}\bar{\gamma}_{k}}{\bar{\gamma}_{k+1}}, \label{eq:RDSB_prop_2_a_lemma_1} \\
            & \tilde{\mu}^n_k(x_k',x_N')=x_k'+\frac{\gamma_{k+1}}{1-\bar{\gamma}_{k}}(x_N'-x_k')+\tilde{l}(\bar{\gamma}_{k},\bar{\gamma}_{k+1})-\frac{\gamma_{k+1}}{1-\bar{\gamma}_{k}}\tilde{l}(\bar{\gamma}_{k},1), \nonumber \\
            & \tilde{\sigma}_{k+1}=\frac{2\gamma_{k+1}(1-\bar{\gamma}_{k+1})}{1-\bar{\gamma}_k}. \label{eq:RDSB_prop_2_b_lemma_1}
        \end{align}
        \label{eq:RDSB_prop_2_lemma_1}
    \end{subequations}
    \label{lem:RDSB_prop_1_lemma_1}
\end{lemma}

\begin{proof}
    We first prove Equation~\ref{eq:RDSB_prop_1_a_lemma_1} and Equation~\ref{eq:RDSB_prop_2_a_lemma_1}. Denote for simplicity
    \begin{equation}
        \begin{aligned}
            &\tilde{h}_1=\tilde{h}(\bar{\gamma}_k,\bar{\gamma}_{k+1})=\int_{\bar{\gamma}_k}^{\bar{\gamma}_{k+1}}\left(f(x_t,t)+g^2(t)\nabla\log\mathbf{\Psi}(x_t,t)\right)\mathrm{d}t, \\
            &\tilde{h}_2=\tilde{h}(\bar{\gamma}_0,\bar{\gamma}_{k})=\int_{\bar{\gamma}_0}^{\bar{\gamma}_{k}}\left(f(x_t,t)+g^2(t)\nabla\log\mathbf{\Psi}(x_t,t)\right)\mathrm{d}t, \\
            &\tilde{h}_3=\tilde{h}(\bar{\gamma}_0,\bar{\gamma}_{k+1})=\int_{\bar{\gamma}_0}^{\bar{\gamma}_{k+1}}\left(f(x_t,t)+g^2(t)\nabla\log\mathbf{\Psi}(x_t,t)\right)\mathrm{d}t, 
        \end{aligned}
    \end{equation}
    and we have
    \begin{equation}
        \begin{aligned}
            p_{k+1|k}^n(x_{k+1}|x_{k})&=\mathcal{N}\left(x_{k+1};x_k+\tilde{h}_1,2\gamma_{k+1}\mathbf{I}\right), \\
            p_{k|0}^n(x_{k}|x_{0})&=\mathcal{N}\left(x_{k};x_0+\tilde{h}_2,2\bar{\gamma}_{k}\mathbf{I}\right), \\
            p_{k+1|0}^n(x_{k+1}|x_{0})&=\mathcal{N}\left(x_{k+1};x_0+\tilde{h}_3,2\bar{\gamma}_{k+1}\mathbf{I}\right),
        \end{aligned}
    \end{equation}
    \begin{equation}
        \begin{aligned}
            &\ p_{k|k+1,0}^n(x_{k}|x_{k+1},x_0) \\
            =&\ \frac{p_{k+1|k}^n(x_{k+1}|x_{k})p_{k|0}^n(x_{k}|x_{0})}{p_{k+1|0}^n(x_{k+1}|x_{0})} \\
            =&\ \frac{\mathcal{N}\left(x_{k+1};x_k+\tilde{h}_1,2\gamma_{k+1}\mathbf{I}\right)\mathcal{N}\left(x_{k};x_0+\tilde{h}_2,2\bar{\gamma}_{k}\mathbf{I}\right)}{\mathcal{N}\left(x_{k+1};x_0+\tilde{h}_3,2\bar{\gamma}_{k+1}\mathbf{I}\right)} \\
            \propto&\ \exp\left\{-\frac{1}{2}\left[\frac{\left(x_{k+1}-x_k-\tilde{h}_1\right)^2}{2\gamma_{k+1}}+\frac{\left(x_{k}-x_0-\tilde{h}_2\right)^2}{2\bar{\gamma}_{k}}-\frac{\left(x_{k+1}-x_0-\tilde{h}_3\right)^2}{2\bar{\gamma}_{k+1}}\right]\right\} \\
            =&\ \exp\left\{-\frac{1}{4}\left[\left(\frac{1}{\gamma_{k+1}}+\frac{1}{\bar{\gamma}_{k}}\right)x_k^2-2\left(\frac{x_{k+1}-\tilde{h}_1}{\gamma_{k+1}}+\frac{x_0+\tilde{h}_2}{\bar{\gamma}_k}\right)x_k+C_1\right]\right\} \\
            =&\ \exp\left\{-\frac{1}{2}\left[\frac{\left[x_k-\left(\frac{\bar{\gamma}_k}{\bar{\gamma}_{k+1}}(x_{k+1}-\tilde{h}_1)+\frac{\gamma_{k+1}}{\bar{\gamma}_{k+1}}(x_{0}+\tilde{h}_2)\right)\right]^2}{\frac{2\gamma_{k+1}\bar{\gamma}_k}{\bar{\gamma}_{k+1}}}+C_2\right]\right\} \\
            =&\ \exp\left\{-\frac{1}{2}\left[\frac{\left[x_k-\left(x_{k+1}+\frac{\gamma_{k+1}}{\bar{\gamma}_{k+1}}(x_0-x_{k+1})-\tilde{h}_1+\frac{\gamma_{k+1}}{\bar{\gamma}_{k+1}}\tilde{h}_3\right)\right]^2}{\frac{2\gamma_{k+1}\bar{\gamma}_k}{\bar{\gamma}_{k+1}}}+C_2\right]\right\} \\
            \propto&\ \mathcal{N}(x_k;x_{k+1}+\frac{\gamma_{k+1}}{\bar{\gamma}_{k+1}}(x_0-x_{k+1})-\tilde{h}_1+\frac{\gamma_{k+1}}{\bar{\gamma}_{k+1}}\tilde{h}_3,\frac{2\gamma_{k+1}\bar{\gamma}_k}{\bar{\gamma}_{k+1}}\mathbf{I}),
        \end{aligned}
    \end{equation}
    where $C_1$ and $C_2$ are constants up to $\{x_{0},x_{k+1},\tilde{h}_1,\tilde{h}_2,\tilde{h}_3\}$. We can also prove Equation~\ref{eq:RDSB_prop_1_b_lemma_1} and Equation~\ref{eq:RDSB_prop_2_b_lemma_1} likewise.
\end{proof}

\begin{lemma}
    Assume that
    \begin{equation}
        \begin{aligned}
            \frac{1}{\bar{\gamma}_{k}}\tilde{h}(0,\bar{\gamma}_k)\approx\frac{1}{\bar{\gamma}_{k+1}}\tilde{h}(0,\bar{\gamma}_{k+1}), \frac{1}{1-\bar{\gamma}_{k}}\tilde{l}(\bar{\gamma}_{k},1)\approx\frac{1}{1-\bar{\gamma}_{k+1}}\tilde{l}(\bar{\gamma}_{k+1},1).
        \end{aligned}
    \end{equation}
    Then, we have
    \begin{subequations}
        \begin{align}
            &\ \tilde{h}(\bar{\gamma}_k,\bar{\gamma}_{k+1})-\frac{\gamma_{k+1}}{\bar{\gamma}_{k+1}}\tilde{h}(0,\bar{\gamma}_{k+1})\approx 0, \label{eq:RDSB_prop_1_lemma_2_a} \\
            &\ \tilde{l}(\bar{\gamma}_{k},\bar{\gamma}_{k+1})-\frac{\gamma_{k+1}}{1-\bar{\gamma}_{k}}\tilde{l}(\bar{\gamma}_{k},1)\approx 0.  \label{eq:RDSB_prop_1_lemma_2_b}
        \end{align}
        \label{eq:RDSB_prop_1_lemma_2}
    \end{subequations}
    \label{lem:RDSB_prop_1_lemma_2}
\end{lemma}

\begin{proof}
    We first prove Equation~\ref{eq:RDSB_prop_1_lemma_2_a}. We have
    \begin{equation}
        \begin{aligned}
            &\ \tilde{h}(\bar{\gamma}_k,\bar{\gamma}_{k+1})-\frac{\gamma_{k+1}}{\bar{\gamma}_{k+1}}\tilde{h}(0,\bar{\gamma}_{k+1}) \\
            =&\ \bar{\gamma}_k\left(\frac{1}{\bar{\gamma}_k}\tilde{h}(\bar{\gamma}_k,\bar{\gamma}_{k+1})-\left(\frac{1}{\bar{\gamma}_k}-\frac{1}{\bar{\gamma}_{k+1}}\right)\tilde{h}(0,\bar{\gamma}_{k+1})\right) \\
            =&\ \bar{\gamma}_k\left(-\frac{1}{\bar{\gamma}_k}\tilde{h}(0,\bar{\gamma}_{k})+\frac{1}{\bar{\gamma}_{k+1}}\tilde{h}(0,\bar{\gamma}_{k+1})\right) \\
            \approx&\ 0.
        \end{aligned}
    \end{equation}
    We can also prove Equation~\ref{eq:RDSB_prop_1_lemma_2_b} likewise.
\end{proof}

Lemma~\ref{lem:RDSB_prop_1_lemma_0}, Lemma~\ref{lem:RDSB_prop_1_lemma_1} and Lemma~\ref{lem:RDSB_prop_1_lemma_2} together conclude the proof of Proposition~\ref{prop:RDSB_prop_full}.

\subsection{Reparameterization}

According to the forward and backward propagation formulas of DSB, we have:

\begin{equation}
    \begin{aligned}
        &x_{\bar{\gamma}_{k+1}} = F_{\bar{\gamma}_k}^n(x_{\bar{\gamma}_k})+\sqrt{2\gamma_{k+1}}\epsilon = x_{\bar{\gamma}_k}+\gamma_{k+1}f_{\bar{\gamma}_k}^n(x_{\bar{\gamma}_k})+\sqrt{2\gamma_{k+1}}\epsilon, \\
        &x_{\bar{\gamma}_{k}} = B_{\bar{\gamma}_{k+1}}^n(x_{\bar{\gamma}_{k+1}})+\sqrt{2\gamma_{k+1}}\epsilon = x_{\bar{\gamma}_{k+1}}+\gamma_{k+1}b_{\bar{\gamma}_{k+1}}^n(x_{\bar{\gamma}_{k+1}})+\sqrt{2\gamma_{k+1}}\epsilon.
    \end{aligned}
    \label{eq:DSB_target_1}
\end{equation}

\noindent By Equation~\ref{eq:SB_SDE}, the DSB aims to solve:

\begin{equation}
    \begin{aligned}
        x_{\bar{\gamma}_{k+1}} &= x_{\bar{\gamma}_k}+\int_{\bar{\gamma}_k}^{\bar{\gamma}_{k+1}}\left[\left(f(x_t,t)+g^2(t)\nabla\log\mathbf{\Psi}(x_t,t)\right)\mathrm{d}t+g(t)\mathrm{d}\mathbf{W}_t\right] \\
        &= x_{\bar{\gamma}_k}+\int_{\bar{\gamma}_k}^{\bar{\gamma}_{k+1}}\left(f(x_t,t)+g^2(t)\nabla\log\mathbf{\Psi}(x_t,t)\right)\mathrm{d}t+\sqrt{2\gamma_{k+1}}\epsilon, \\
        x_{\bar{\gamma}_{k}} &= x_{\bar{\gamma}_{k+1}}+\int_{\bar{\gamma}_{k+1}}^{\bar{\gamma}_{k}}\left[\left(f(x_t,t)-g^2(t)\nabla\log\mathbf{\hat{\Psi}}(x_t,t)\right)\mathrm{d}t+g(t)\mathrm{d}\mathbf{\bar{W}}_t\right]\\
        &= x_{\bar{\gamma}_{k+1}}+\int_{\bar{\gamma}_{k+1}}^{\bar{\gamma}_{k}}\left(f(x_t,t)-g^2(t)\nabla\log\mathbf{\hat{\Psi}}(x_t,t)\right)\mathrm{d}t+\sqrt{2\gamma_{k+1}}\epsilon,
    \end{aligned}
    \label{eq:DSB_target_2}
\end{equation}

\noindent where $g(t)\equiv\sqrt{2}$ in~\cite{de2021diffusion} and $\epsilon\sim\mathcal{N}(0,1)$. By comparing Equation~\ref{eq:DSB_target_1} and Equation~\ref{eq:DSB_target_2}, we find that DSB essentially learns

\begin{equation}
    \begin{aligned}
        f_{\bar{\gamma}_{k}}^n(x_{\bar{\gamma}_{k}})&\approx\frac{1}{\gamma_{k+1}}\int_{\bar{\gamma}_{k}}^{\bar{\gamma}_{k+1}}\left(f(x_t,t)+g^2(t)\nabla\log\mathbf{\Psi}(x_t,t)\right)\mathrm{d}t, \\
        b_{\bar{\gamma}_{k+1}}^n(x_{\bar{\gamma}_{k+1}})&\approx\frac{1}{\gamma_{k+1}}\int_{\bar{\gamma}_{k+1}}^{\bar{\gamma}_{k}}\left(f(x_t,t)-g^2(t)\nabla\log\mathbf{\hat{\Psi}}(x_t,t)\right)\mathrm{d}t,
    \end{aligned}
    \label{eq:DSB_sde_target}
\end{equation}

\noindent where Equation~\ref{eq:SB_SDE} can be viewed as the limit at $\gamma_{k+1} \rightarrow 0$. However, $\mathrm{d}\mathbf{X}_t$ can be noisy to learn as $\mathbf{X}_t$ is sampled from SDE, which involves iteratively adding Gaussian noise, and the network is hard to converge in our practical settings. Given Proposition~\ref{prop:RDSB_prop_full} and the observations in Equation~\ref{eq:DSB_sde_target}, we propose two reparameterization methods inspired by prevalent SGMs in Section~\ref{sec:RDSB}. Here we provide some intuitive understandings.

\textbf{Terminal Reparameterization (TR)} We use two neural networks to learn $\tilde{F}_{\alpha^n}(\bar{\gamma}_k,x_{\bar{\gamma}_k})\approx\tilde{F}_{\bar{\gamma}_k}^n(x_{\bar{\gamma}_k})$ and $\tilde{B}_{\beta^n}(\bar{\gamma}_{k+1},x_{\bar{\gamma}_{k+1}})\approx\tilde{B}_{\bar{\gamma}_{k+1}}^n(x_{\bar{\gamma}_{k+1}})$ s.t.

\begin{equation}
    \begin{aligned}
        \tilde{B}_{\bar{\gamma}_{k+1}}^n(x_{\bar{\gamma}_{k+1}})&\approx x_0=x_{\bar{\gamma}_{k+1}}+\int_{\bar{\gamma}_{k+1}}^{\bar{\gamma}_0}\left(f(x_t,t)-g^2(t)\nabla\log\mathbf{\hat{\Psi}}(x_t,t)\right)\mathrm{d}t, \\
        \tilde{F}_{\bar{\gamma}_k}^n(x_{\bar{\gamma}_k})&\approx x_1=x_{\bar{\gamma}_k}+\int_{\bar{\gamma}_k}^{\bar{\gamma}_N}\left(f(x_t,t)+g^2(t)\nabla\log\mathbf{\Psi}(x_t,t)\right)\mathrm{d}t, 
    \end{aligned}
\end{equation}

\noindent and the training loss is Equation~\ref{eq:TRDSB_loss}. Intuitively, the networks learn to predict terminal points $x_0\sim p_\textup{data}$ and $x_1\sim p_\textup{prior}$ given $x_t$. This recovers DDPM~\cite{ho2020denoising} when $\nabla\log\mathbf{\Psi}(x_t,t)=0$.

\textbf{Flow Reparameterization (FR)} We use two neural networks to learn $\tilde{f}_{\alpha^n}(\bar{\gamma}_k,x_{\bar{\gamma}_k})\approx\tilde{f}_{\bar{\gamma}_k}^n(x_{\bar{\gamma}_k})$ and $\tilde{b}_{\beta^n}(\bar{\gamma}_{k+1},x_{\bar{\gamma}_{k+1}})\approx\tilde{b}_{\bar{\gamma}_{k+1}}^n(x_{\bar{\gamma}_{k+1}})$ s.t.

\begin{equation}
    \begin{aligned}
        \tilde{b}_{\bar{\gamma}_{k+1}}^n(x_{\bar{\gamma}_{k+1}})&\approx\frac{1}{\bar{\gamma}_{k+1}-\bar{\gamma}_0}\int_{\bar{\gamma}_{k+1}}^{\bar{\gamma}_0}\left(f(x_t,t)-g^2(t)\nabla\log\mathbf{\hat{\Psi}}(x_t,t)\right)\mathrm{d}t, \\
        \tilde{f}_{\bar{\gamma}_k}^n(x_{\bar{\gamma}_k})&\approx\frac{1}{\bar{\gamma}_N-\bar{\gamma}_k}\int_{\bar{\gamma}_k}^{\bar{\gamma}_N}\left(f(x_t,t)+g^2(t)\nabla\log\mathbf{\Psi}(x_t,t)\right)\mathrm{d}t,
    \end{aligned}
\end{equation}

\noindent and the training loss is Equation~\ref{eq:FRDSB_loss}. Intuitively, the networks learn the vector field that points from $x_t$ to $x_0$ and $x_1$. This recovers Flow Matching~\cite{lipman2022flow} when $\nabla\log\mathbf{\Psi}(x_t,t)=0$.

In summary, R-DSB learns the integration of Equation~\ref{eq:SB_SDE}.

\section{General Framework of Dynamic Generative Models}

The field of dynamic generative models is currently a subject of intense scholarly interest and researchers have proposed different methods, including Score-based Generative Models (SGMs)~\cite{ho2020denoising,song2019generative}, Flow Matching (FM)~\cite{lipman2022flow}, Image-to-Image Schrödinger Bridge (I$^2$SB)~\cite{liu20232}, Bridge-TTS (BTTS)~\cite{chen2023schrodinger} and Diffusion Schrödinger Bridge (DSB)~\cite{de2021diffusion}. We find that these methods share similarities and can be summarized with Algorithm~\ref{alg:unified_training} and Table~\ref{tab:function_choice}.

In practice, we find that these approaches share a general framework of training, differ only by the network prediction target and how to sample noised data $x_{k}$. These methods can all be summarized by Equation~\ref{eq:SB_SDE}, and to a certain extent can be regarded as special cases of solving the Schrödinger Bridge (SB) problem. For example, if we set the non-linear drift $\nabla\log\mathbf{\Psi}(\mathbf{X}_t,t)$ to zero, by Nelson's duality~\cite{nelson1967dynamical} and Equation~\ref{eq:SB_marginal} we have

\begin{equation}
    \begin{aligned}
        \nabla\log\mathbf{\hat{\Psi}}(\mathbf{X}_t,t) &= -\nabla\log\mathbf{\Psi}(\mathbf{X}_t,t) + \nabla\log p_t(\mathbf{X}_t) \\
        &= \nabla\log p_t(\mathbf{X}_t),
    \end{aligned}
\end{equation}

\noindent where $p_t$ is the marginal density at time $t$, which is indeed the score in SGM. This unification motivates us to integrate beneficial designs of SGMs into our Simplified DSB (S-DSB) and Reparameterized DSB (R-DSB), and may inspire future works.

Finally, the training pipeline of our S-DSB is summarized in Algorithm~\ref{alg:SDSB}, where we highlight all differences with the original DSB~\cite{de2021diffusion} in blue.

\begin{algorithm}[!t]
	\caption{Training dynamic generative models} 
    \hspace*{\algorithmicindent} \textbf{Input}: Rounds $R$, Timesteps $N$, Data distribution $p_\textup{data}$, Prior distribution $p_\textup{prior}$, Learning rate $\eta$ \\
    \hspace*{\algorithmicindent} \textbf{Output}: Model $M_\theta$
	\begin{algorithmic}[1]
		\For {$n\in\{1,\ldots,R\}$}
			\While {not converged}
				\State Sample $x_0\sim p_\textup{data}$, $x_N\sim p_\textup{prior}$ and $k\in\{0,1,\ldots,N-1\}$
                \State Get noised sample $x_k$
                \State Get prediction target $y_k$
				\State Compute loss $\mathcal{L}=\left\|M_\theta(k,x_k)-y_k\right\|$
                \State Optimize $\theta\leftarrow\theta-\eta\nabla_\theta\mathcal{L}$
			\EndWhile
		\EndFor
	\end{algorithmic}
    \label{alg:unified_training}
\end{algorithm}

\begin{table}[!t]
    \centering
    \begin{tabular}{@{}l|ccc@{}}
    \toprule
                  & $R$ & $x_k$ & $y_k$ \\ \midrule
    SGM~\cite{song2019generative}     & 1 & $x_0+\beta_kx_N$ & $\nabla\log p_k$ \\
    SGM~\cite{ho2020denoising}     & 1 & $\sqrt{\bar{\alpha}_k}x_0+\sqrt{1-\bar{\alpha}_k}x_N$ & $x_N$ \\
    FM~\cite{lipman2022flow} & 1 & $(1-\frac{k}{N})x_0+\frac{k}{N}x_N$ & $x_N-x_0$ \\
    I$^2$SB~\cite{liu20232} & 1 & $\frac{\bar{\sigma}_k^2}{\bar{\sigma}_k^2+\sigma_k^2}x_0+\frac{\sigma_k^2}{\bar{\sigma}_k^2+\sigma_k^2}x_N+\frac{\sigma_k\bar{\sigma}_k}{\sqrt{\bar{\sigma}_k^2+\sigma_k^2}}\epsilon$ & $\frac{1}{\sigma_k}(x_k-x_0)$ \\
    BTTS~\cite{chen2023schrodinger} & 1 & $\frac{\alpha_k\bar{\sigma}_k^2}{\sigma_N^2}x_0+\frac{\bar{\alpha}_k\sigma_k^2}{\sigma_N^2}x_N+\frac{\alpha_k\bar{\alpha}_k\sigma_k}{\sigma_N}\epsilon$ & $x_0$ \\
    DSB~\cite{de2021diffusion} & $>1$ & $\left.\begin{matrix}F_{k-1}(x_{k-1})+\sqrt{2\gamma_k}\epsilon|_{n=1,3,\ldots}\\B_{k+1}(x_{k+1})+\sqrt{2\gamma_k}\epsilon|_{n=2,4,\ldots}\end{matrix}\right.$ & $\left.\begin{matrix}x_k+F_{k-1}(x_{k-1})-F_{k-1}(x_k)|_{n=1,3,\ldots}\\x_k+B_{k+1}(x_{k+1})-B_{k+1}(x_k)|_{n=2,4,\ldots}\end{matrix}\right.$ \\
    S-DSB & $>1$ & $\left.\begin{matrix}F_{k-1}(x_{k-1})+\sqrt{2\gamma_k}\epsilon|_{n=1,3,\ldots}\\B_{k+1}(x_{k+1})+\sqrt{2\gamma_k}\epsilon|_{n=2,4,\ldots}\end{matrix}\right.$ & $\left.\begin{matrix}x_{k-1}|_{n=1,3,\ldots}\\x_{k+1}|_{n=2,4,\ldots}\end{matrix}\right.$ \\
    TR-DSB & $>1$ & $\left.\begin{matrix}F_{k-1}(x_{k-1})+\sqrt{2\gamma_k}\epsilon|_{n=1,3,\ldots}\\B_{k+1}(x_{k+1})+\sqrt{2\gamma_k}\epsilon|_{n=2,4,\ldots}\end{matrix}\right.$ & $\left.\begin{matrix}x_{0}|_{n=1,3,\ldots}\\x_{N}|_{n=2,4,\ldots}\end{matrix}\right.$ \\
    FR-DSB & $>1$ & $\left.\begin{matrix}F_{k-1}(x_{k-1})+\sqrt{2\gamma_k}\epsilon|_{n=1,3,\ldots}\\B_{k+1}(x_{k+1})+\sqrt{2\gamma_k}\epsilon|_{n=2,4,\ldots}\end{matrix}\right.$ & $\left.\begin{matrix}\frac{1}{\bar{\gamma}_k-\bar{\gamma}_0}(x_{k}-x_{0})|_{n=1,3,\ldots}\\\frac{1}{\bar{\gamma}_N-\bar{\gamma}_k}(x_{N}-x_{k})|_{n=2,4,\ldots}\end{matrix}\right.$ \\
    \bottomrule
    \end{tabular}
    \caption{Specific choices of different methods on training rounds $R$, noised sample $x_k$, and prediction target $y_k$ in Algorithm~\ref{alg:unified_training}, where $\epsilon\sim\mathcal{N}(0,1)$, $x_0\sim p_\textup{data}$, and $x_N\sim p_\textup{prior}$.}
    \label{tab:function_choice}
\end{table}

\begin{algorithm}[!h]
    \caption{Simplified Diffusion Schrödinger Bridge}
    \begin{algorithmic}[1] 
      \For {$n\in\{0,\dots,L\}$}
        \While {not converged}
            \State Sample $\{X^j_{k}\}_{k,j=0}^{N,M}$, where  $X^j_0 \sim p_\textup{data}$, and $X^{j}_{k+1} = F_{\alpha^n}(k, X^{j}_{k})+\sqrt{2\gamma_{k+1}}\epsilon$
            \State \textcolor{blue}{$\hat{\ell}_n^b\leftarrow\nabla_{\beta^n}\left[\left\|B_{\beta^n}(k+1,X^{j}_{k+1})-X^{j}_{k}\right\|^2\right]$}
            \State $\beta^{n} \leftarrow \textrm{Gradient Step}(\hat{\ell}_n^b(\beta^n))$ 
        \EndWhile
        \While {not converged}
            \State Sample $\{X^j_{k}\}_{k,j=0}^{N,M}$, where $X^j_N \sim p_\textup{prior}$, and $X^j_{k-1}=B_{\beta^n}(k, X^{j}_k)+\sqrt{2\gamma_{k}}\epsilon$ 
            \State \textcolor{blue}{$\hat{\ell}_{n+1}^f\leftarrow\nabla_{\alpha^{n+1}}\left[\left\|F_{\alpha^{n+1}}(k,X^{j}_{k})-X^{j}_{k+1}\right\|^2\right]$}
            \State $\alpha^{n+1} \leftarrow \textrm{Gradient Step}(\hat{\ell}_{n+1}^f(\alpha^{n+1}))$
        \EndWhile
      \EndFor
    \end{algorithmic}
    \label{alg:SDSB}
\end{algorithm}

\section{Limitation}

Although we propose a simplified training target in Equation~\ref{eq:DSB_simplified_loss} and two reparameterization techniques in Section~\ref{sec:RDSB}, we find that the convergence of DSB is still time-consuming. For example, it takes about one day and $8\times$ V100 GPUs to train an unpaired image-to-image translation model on the AFHQ \texttt{cat-dog} dataset and $256\times256$ resolution, even provided two pretrained Flow Matching models as initialization. It also hinders us from scaling up experiments on larger datasets. Besides, our R-DSB is based on the assumption of Equation~\ref{eq:RDSB_assumption}, which may cause errors in practical applications. While R-DSB presents promising performances, we believe better approximation and error analysis will be a valuable direction for future research.

\section{More Results}

In Figure~\ref{fig:results_celeba}, we illustrate unconditional generation results on the CelebA~\cite{liu2015faceattributes} dataset at $64\times64$ resolution. Our method is capable of generating images from Gaussian priors ($p_\textup{prior}=\mathcal{N}(0,1)$).

We show more results on the AFHQ dataset at $256\times256$ resolution in Figure~\ref{fig:dsb-supp-256}. Moreover, we also scale up the image resolution to $512\times512$ in Figure~\ref{fig:dsb-supp-512}. It demonstrates that our method generates realistic appearances and rich details on high-dimensional space.

\begin{figure}[t]
    \centering
    \includegraphics[width=\linewidth]{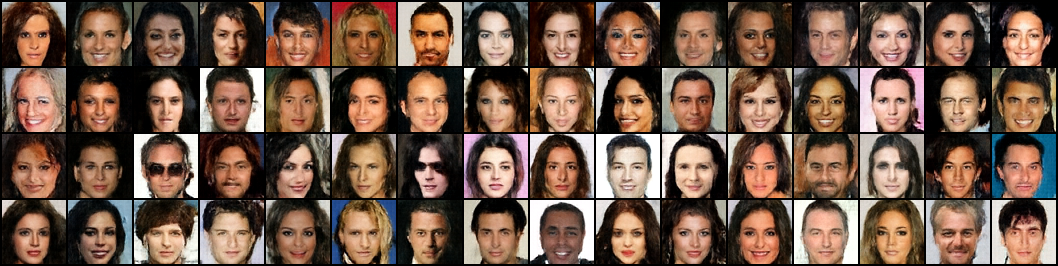}
    \caption{Generation results on CelebA.}
    \label{fig:results_celeba}
\end{figure}

\begin{figure}[t]
    \centering
    \includegraphics[width=\textwidth]{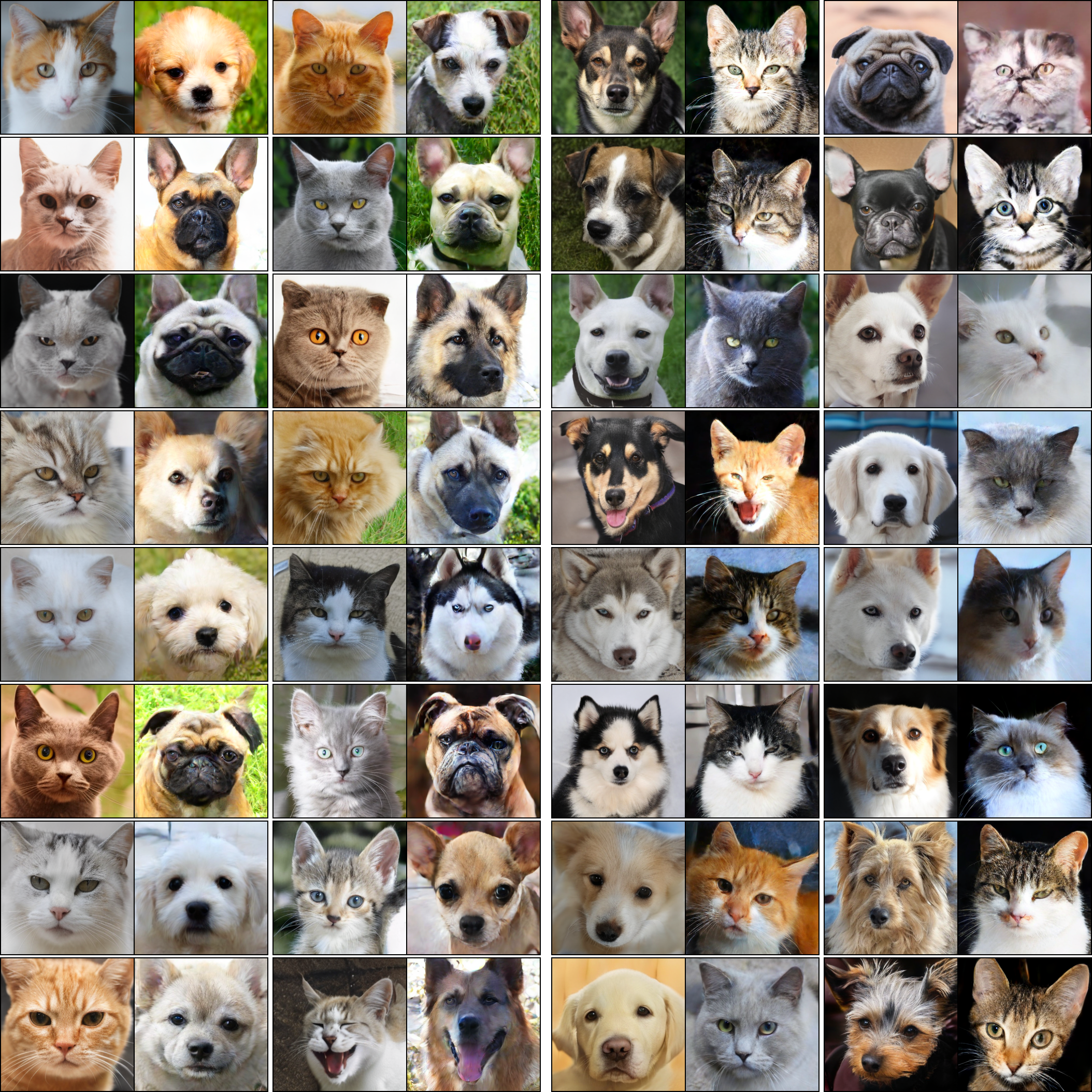}
    \caption{More results at $256\times256$ resolution on unpaired \texttt{dog} $\leftrightarrow$ \texttt{cat} translation. Left: \texttt{cat} to \texttt{dog}; right: \texttt{dog} to \texttt{cat}. For each case, the left is the input image and the right is the translated result from our model.}
    \label{fig:dsb-supp-256}
\end{figure}

\begin{figure}[t]
    \centering
    \includegraphics[width=\textwidth]{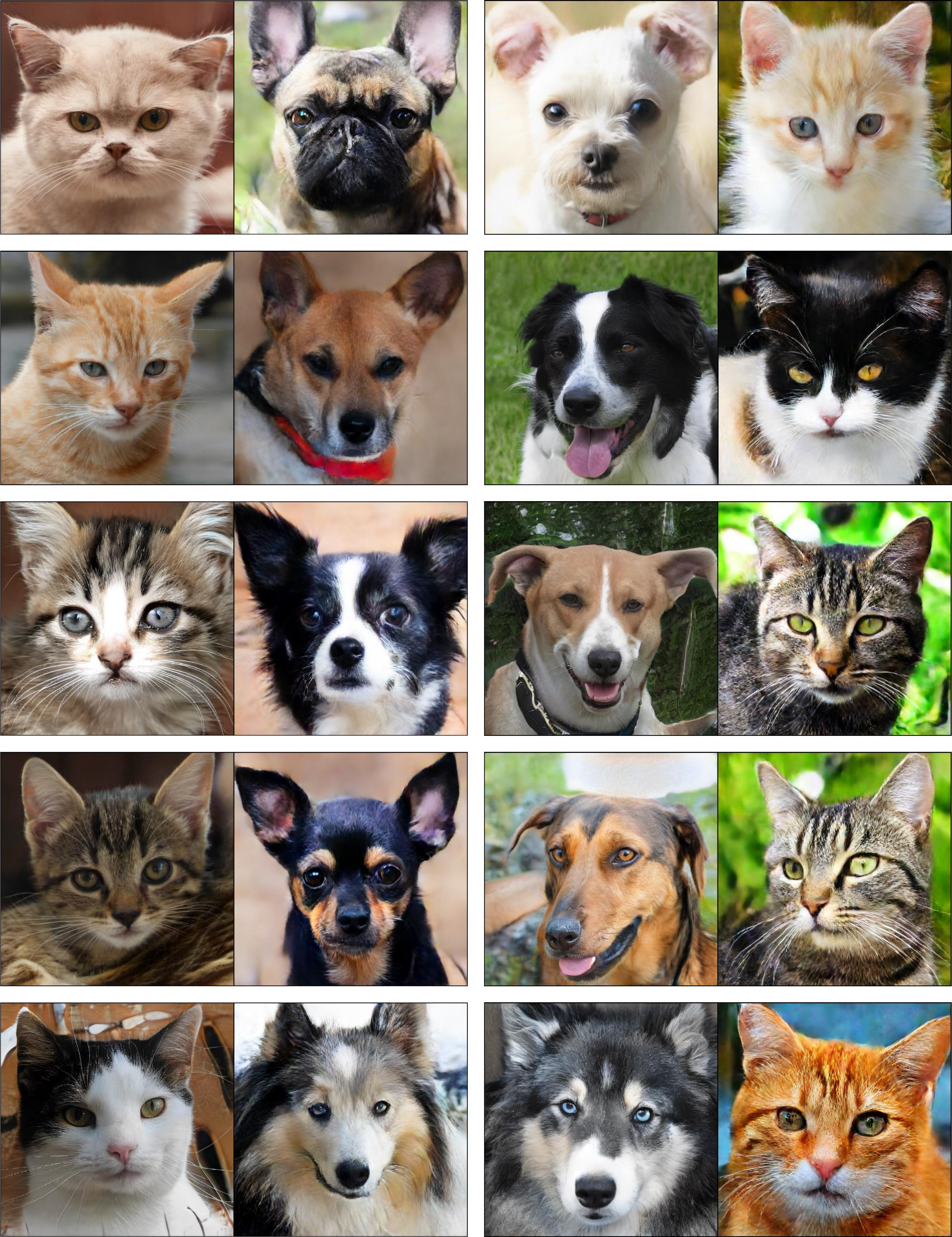}
    \caption{More results at $512\times512$ resolution on unpaired \texttt{dog} $\leftrightarrow$ \texttt{cat} translation. Left: \texttt{cat} to \texttt{dog}; right: \texttt{dog} to \texttt{cat}. For each case, the left is the input image and the right is the translated result from our model.}
    \label{fig:dsb-supp-512}
\end{figure}

\end{document}